\documentclass{article}

\usepackage[final]{neurips_2022}

\usepackage[utf8]{inputenc} 
\usepackage[T1]{fontenc}    
\usepackage{hyperref}       
\usepackage{url}            
\usepackage{booktabs}       
\usepackage{amsfonts}       
\usepackage{nicefrac}       
\usepackage{microtype}      
\usepackage{xcolor}         

\usepackage{algorithm}
\usepackage{algorithmic}

\usepackage{comment}
\usepackage{subcaption}

\usepackage{amsmath}
\usepackage{amsthm}
\usepackage{amsfonts}
\usepackage{nccmath}
\usepackage{bbm}
\usepackage{mathtools}
\usepackage{natbib}
\usepackage{url}

\newtheorem{thm}{Theorem}
\newtheorem{cor}[thm]{Corollary}
\newtheorem{lem}[thm]{Lemma}

\newtheorem{prop}[thm]{Proposition}

\newcommand{\RR}{ \mathbb{R} }

\newcommand{\setsep}{\;\middle|\;}
\newcommand{\half}{\frac{1}{2}}

\newcommand{\Prob}[1]{\mathbb{P}\left( #1 \right)}

\newcommand{\Abs}[1]{\left| #1 \right|}
\newcommand{\Set}[1]{\left\{ #1 \right\}}
\newcommand{\Brack}[1]{\left( #1 \right)}
\newcommand{\BBrack}[1]{\left\{ #1 \right\}}

\newcommand{\inner}[2]{\left< #1 , #2 \right>}

\newcommand{\Exp}[1]{ \mathbb{E} #1}

\newcommand{\norms}[1]{\left|#1\right|}

\newcommand{\normop}[1]{\left\|#1\right\|_{op}}
\newcommand{\normopnuc}[1]{\left\|#1\right\|_{nuc}}
\newcommand{\normophs}[1]{\left\|#1\right\|_{HS}}

\newcommand{\tr}{tr}

\DeclarePairedDelimiter\ceil{\lceil}{\rceil}

\title{Finite Sample Analysis Of Dynamic Regression Parameter Learning}

\author{
Mark Kozdoba \\ 
Technion, Israel Institute of Technology \\
\texttt{markk@ef.technion.ac.il} \\
\And 
Edward Moroshko \\
Technion, Israel Institute of Technology \\
\texttt{edward.moroshko@gmail.com}
\AND
Shie Mannor \\
Technion, Israel Institute of Technology \\
\texttt{shie@ee.technion.ac.il}
\And 
Koby Crammer\\
Technion, Israel Institute of Technology \\
\texttt{koby@ee.technion.ac.il}
}

\newif\ifimagesshow
\imagesshowtrue

\begin{document}

\maketitle

\begin{abstract}
We consider the dynamic linear regression problem, where the predictor vector may vary with time. This problem can be modeled as a linear dynamical system, with non-constant observation operator, where the parameters that need to be learned are the variance of both the process noise and the observation noise. While variance estimation for dynamic regression is a natural problem, with a variety of applications, existing approaches to this problem either lack guarantees altogether, or only have asymptotic guarantees without explicit rates. In particular, existing literature does not provide any clues to the following  fundamental question: In terms of data characteristics, what does the convergence rate depend on?  In this paper we study the global system operator -- the operator that maps the  noise vectors to the output. We obtain estimates on its spectrum, and as a result derive the first known variance estimators with finite sample complexity guarantees. The proposed bounds depend on the shape of a certain spectrum related to the system operator, and thus provide the first known explicit geometric parameter of the data that can be used to bound estimation errors. In addition, the results hold for arbitrary sub Gaussian distributions of noise terms.  We evaluate the approach on synthetic and real-world benchmarks.
\end{abstract}

\newcommand{\basesystem}{(\ref{eq:lds_main_1})-(\ref{eq:lds_main_2}) }

\section{Introduction}
\label{sec:introduction}
A dynamic linear regression \citep[Chapter 3]{WestHarrison}, or non-stationary regression, is a situation where we are given 
a sequence of scalar \textit{observations} $\Set{Y_t}_{t\leq T}\subset \RR$, 
and \textit{observation vectors} $\Set{u_t}_{t\leq T}\subset \RR^{n}$ such that 
$Y_t = \inner{X_t}{u_t} + z_t$ where $X_t\in \RR^n$ is a regressor vector, and 
$z_t$ a random noise term. In contrast to a standard linear regression, the vector 
$X_t$ may change with time. One common objective for this problem is at time $T$, to estimate the trajectory of $X_t$
for $t\leq T$, given the observation vectors and observations, $\Set{u_t}_{t\leq T}$, 
$\Set{Y_t}_{t\leq T}$, and possibly to forecast $Y_{T+1}$ if $u_{T+1}$ is also known.

In this paper we model the problem as follows: 
\begin{eqnarray}
\label{eq:lds_main_1}
   X_{t+1} &=& X_t + h_t \\
\label{eq:lds_main_2}
   Y_t &=& \inner{X_t}{u_t} + z_t\,,
\end{eqnarray}
where $\inner{\cdot}{\cdot}$ is the standard inner product on $\RR^n$, $z_t$, the \emph{observation noise}, are zero-mean sub Gaussian random variables, 
with variance $\eta^2$, and the \emph{process noise} variables 
 $h_t$ take values in $\RR^n$, 
 such that coordinates of $h_t$ are zero-mean sub Gaussian, independent, and have  variance $\sigma^2$. All $h_t$ and $z_t$ variables are assumed to be mutually independent. 
The vectors $u_t$ are an \textit{arbitrary} sequence 
in $\RR^n$, and the observed, known,  quantities at time $T$ are  
$\Set{Y_t}_{t\leq T}$ and $\Set{u_t}_{t\leq T}$.

The system 
(\ref{eq:lds_main_1})-(\ref{eq:lds_main_2}) 
is a special case of a Linear Dynamical System (LDS). As is well known, when the 
parameters $\sigma,\eta$ are given, the mean-squared loss optimal forecast for $Y_{T+1}$ and estimate for $X_T$ 
are obtained by the Kalman Filter \citep{anderson1979,hamilton1994time,chui2017kalman}. In this paper we are concerned with estimators for $\sigma,\eta$, and \textit{finite} sample 
complexity guarantees for these estimators.

Let us first make a few remarks  about the particular  
system (\ref{eq:lds_main_1})-(\ref{eq:lds_main_2}).  
First, as a natural model of time varying regression, this  system is useful 
in a considerable variety of applications. We refer to \cite{WestHarrison}, 
Chapter 3, for numerous examples. In addition, an application to electricity 
consumption time series as a function of the temperature is provided in the 
experiments section of this paper. Second, one may regard 
the problem of estimating $\sigma,\eta$ in (\ref{eq:lds_main_1})-(\ref{eq:lds_main_2}) as a pure 
case of finding the optimal \textit{learning rate} for $X_t$. 
Indeed, the Kalman filter equations for (\ref{eq:lds_main_1})-(\ref{eq:lds_main_2}), are given by (\ref{eq:kalman_cov_update})-(\ref{eq:kalman_state_update}) below, where  (\ref{eq:kalman_cov_update})  describes the filtered covariance update and (\ref{eq:kalman_state_update}) the filtered state update. Here $\bar{x}_t$ is the estimated state, given the observations $Y_1,\ldots,Y_{t}$, see \cite{WestHarrison}.  
\begin{align}
\label{eq:kalman_cov_update}
C_{t+1} &=  \frac{\eta^2}{\inner{(C_t+\sigma^2 I)u_{t+1}}{u_{t+1}} + \eta^2}
\Brack{C_t + \sigma^2 I} \\
\label{eq:kalman_state_update}
\bar{x}_{t+1} &= \bar{x}_t +  \frac{ C_{t+1}}{\eta^2}
u_{t+1} \cdot \Brack{Y_{t+1} - \inner{\bar{x}_t}{u_{t+1}}} .
\end{align}
In particular, following (\ref{eq:kalman_state_update}), the role of $\sigma$ and $\eta$ may be interpreted as regulating 
how much the estimate of $\bar{x}_{t+1}$ is influenced, via the operator $\frac{ C_{t+1}}{\eta^2}$, by the most recent observation and input $Y_{t+1}, u_{t+1}$. Roughly speaking, higher values 
of $\sigma$ or lower values of $\eta$ would imply that the past observations are given less weight, and  result in an overfit of the forecast to the most recent observation. On the other hand, very low $\sigma$ or high $\eta$ would make the problem similar to the standard linear 
regression, where all observations are given equal weight, and 
result in a \textit{lag} of the forecast. See Figure \ref{fig:params} in Supplementary Material Section \ref{sec:appendix_outline} for an illustration.

Finally, it is worth mentioning that the system (\ref{eq:lds_main_1})-(\ref{eq:lds_main_2}) is 
closely related to the study of \emph{online gradient} (OG) methods \citep{Zinkevich2003,Hazanbook2016}. In this field, assuming quadratic cost, one considers the update 
\begin{equation}
\label{eq:online_grad_update}
\bar{x}_{t+1} = \bar{x}_t +  \alpha \cdot u_{t+1} \cdot \Brack{Y_{t+1} - \inner{\bar{x}_t}{u_{t+1}}},          
\end{equation}
where $\alpha$ is the learning rate, and studies the performance guarantees of the forecaster $\inner{\bar{x}_t}{u_{t+1}}$. 
Compared to  
(\ref{eq:kalman_state_update}), the update (\ref{eq:online_grad_update}) 
is simpler, and uses a scalar rate $\alpha$ instead of the input-dependent 
operator rate $C_{t+1} \big/ \eta^2$ of the Kalman filter. However, due to 
the similarity, every domain of applicability of the OG methods is also a 
natural candidate for the model (\ref{eq:lds_main_1})-(\ref{eq:lds_main_2}) 
and vice-versa. As an illustration, we compare the OG to Kalman filter based 
methods with learned $\sigma$,$\eta$ in the experiments section.

In this paper we introduce a new estimation algorithm for $\sigma,\eta$, 
termed STVE (Spectrum Thresholding Variance Estimator), and prove finite 
sample complexity bounds for it. In particular, 
our bounds are an explicit function of the parameters $T$ and 
$\Set{u_t}_{t=1}^T$ for any finite $T$, and indicate that the estimation 
error decays roughly as $T^{-\half}$, with high probability. To the best of our knowledge, these are the first bounds of this kind.  As we discuss in detail in Section \ref{sec:literature}, most existing estimation methods for LDSs, such as subspace identification \citep{van1996subspace,qin2006overview}, or improper learning \citep{anava13,hazan2017online,kalman_filter_decay}, do not apply to the system (\ref{eq:lds_main_1})-(\ref{eq:lds_main_2}), due to non-stationarity. 
On the other hand, the methods that do apply to (\ref{eq:lds_main_1})-(\ref{eq:lds_main_2}) either lack guarantees, or have only asymptotic analysis which in addition relies strongly on Guassianity of the noises. 

Moreover, our approach differs significantly from the existing methods. We show that the structure of equations \basesystem  is closely related to, and inherits several important properties from, the classical discrete Laplacian operator on the line --- leading to new arguments that have not been explored in the literature. In particular, we use this connection to show that an appropriate inversion of the system produce estimators that are concentrated enough so that $\sigma$ and $\eta$ may be recovered. The heart of the paper is the new definition of the estimators that exploits explicitly the shape of a certain data dependent operator, and the subsequent concentration analysis. In particular, this approach yields the first known {\em geometric } parameters of the data that can be used to bound convergence rates.

The rest of the paper is organized as follows:   The  related work is discussed in Section \ref{sec:literature} and  Section \ref{sec:notation} contains the necessary definitions. In 
Section \ref{sec:overview}  we describe in general lines the methods and the main results of this paper.  The technical estimates on certain operator spectra, that are critical to the analysis and may be of independent interest, are stated in Section \ref{sec:analysis}. In 
Section \ref{sec:experiments}
we present experimental results on synthetic and real data. Due to space constraints, while we outline the main arguments in the text, the full proofs are deferred to the Supplementary Material.

\section{Literature}
\label{sec:literature}
We refer to \cite{chui2017kalman,hamilton1994time,anderson1979,shumwaystoffer} for a general background on LDSs, the Kalman Filter and 
maximum likelihood estimation.

Existing approaches to the 
variance estimation problem may be divided into three categories: 
(i) General methods for parameter identification in LDS, 
either via maximum likelihood estimation (MLE)  
\citep{hamilton1994time}, or via subspace 
identification \citep{van1996subspace,qin2006overview}. In particular, finite sample bounds for system identification were given in \citep{campi2005guaranteed,vidyasagar2006learning} and in the recent work  \cite{tsiamis2019finite}. (ii) Methods 
designed specifically to learn the noise parameters of 
the system, developed primarily in the control theory community, in particular via the innovation 
auto-correlation function, such as the classical 
\cite{MEHRA70,BELANGER1974267}, or for 
instance more recent \cite{WANG201712,DUNIK201816}. (iii) \emph{Improper Learning} methods, such as \cite{anava13,hazan2017online,kalman_filter_decay}. In these approaches, one does not learn the LDS directly, but instead learns a model from a certain auxiliary class and shows that this auxillary model produces forecasts that are as good as the forecasts of an LDS with ``optimal'' parameters. 
  
Despite the apparent simplicity of the system 
(\ref{eq:lds_main_1})-(\ref{eq:lds_main_2}), most of 
the above methods do not apply to this system. This is 
due to the fact that most of the methods are designed 
for time invariant, asymptotically stationary systems, 
where the observation operator ($u_t$ in our notation) is 
constant and the Kalman gain (or, equivalently $C_t u_t$ in eq. (\ref{eq:kalman_cov_update})) converges with $t$. In particular this limitation exists in all the system identification results cited above, and is essential to the approaches taken there.
However, if the observation vector sequence $u_t$ 
changes with time -- a necessary property for the 
dynamic regression problem -- the system will no longer
be asymptotically stationary.
In particular, due to this reason, neither the subspace identification methods, nor any of the improper learning approaches above apply to system \basesystem. 

Among the methods that do apply to (\ref{eq:lds_main_1})-(\ref{eq:lds_main_2}) are the general MLE estimation, and some of the auto-correlation 
methods \citep{BELANGER1974267,DUNIK201816}.
On one hand, both types of approaches may be 
applicable to systems apriori more general than 
(\ref{eq:lds_main_1})-(\ref{eq:lds_main_2}). On the 
other hand, the situation with consistency guarantees -- the guarantee that one recovers true parameters given enough observations -- 
for these methods is somewhat complicated. Due to the 
non-convexity of the likelihood function, the MLE 
method is not guaranteed to find the true maximum, and 
as a result the whole method has no guarantees. The 
results in \cite{BELANGER1974267,DUNIK201816}  
do have \emph{asymptotic} consistency guarantees. 
However, these rely on some explicit and implicit 
assumptions about the system, the sequence $u_t$ in 
our case, which can not be easily verified. 
In particular, \cite{BELANGER1974267,DUNIK201816} assume \emph{uniform observability} of the system, 
which we do not assume, and in addition rely on certain implicit assumption about invertibility and condition number of 
the matrices related to the sequence $u_t$. Moreover, 
even if one assumes that the assumptions hold, the 
results are purely asymptotic, and for any finite $T$,
do not provide a bound of the expected estimation 
error as a function of $T$ and $\Set{u_t}_{t=1}^T$.

In addition, as mentioned earlier, MLE methods by definition must assume that the noises are Gaussian (or belong to some other  predetermined parametric family) and autocorrelation based methods also strongly use the Gaussianity assumption. Our approach, on the other hand, requires only sub Gaussian noises with independent coordinates. We note that there are straightforward extensions of our methods to certain cases with dependencies. Indeed, the operator analysis part of this paper does not depend on the distribution of the noises. Therefore, to achieve such an extension, one would only need to correspondingly extend the main probabilistic tool, the Hanson-Wright inequality \citep[see also Section \ref{sec:overview} and Supplementary Material Section \ref{sec:main_thm_proof}]{hansonwright,rudelson2013hanson}. One such extension, for vectors with the \emph{convex concentration} property, was recently obtained in \cite{adamczak}. 

\section{Notation}
\label{sec:notation}
  We refer to \cite{bhatia97matrix} and \cite{hdp2018} as general references on the notation introduced below, for operators and sub Gaussian variables, respectively.

Let $A: \RR^n \rightarrow \RR^m$ be an operator with a singular value decomposition
$A = U  \cdot Diag(\lambda_1,\ldots,\lambda_s) \cdot W$,
where $s\leq \min \Set{m,n}$ and $\lambda_1 \geq \lambda_2 \geq \ldots \geq \lambda_s > 0$. Note that 
singular values are strictly positive by definition (that is, vectors corresponding to the kernel of $A$ do not participate in the decomposition $A = U \cdot Diag(\lambda_1,\ldots,\lambda_s) \cdot W$).
The Hilbert-Schmidt (Frobenius) norm is defined as  $\normophs{A} = \sqrt{\sum_{i=1}^s \lambda_i^2}$. The nuclear and the operator norms are given by 
$\normopnuc{A} = \sum_{i=1}^s \lambda_i$ and  
$\normop{A} = \lambda_1$ respectively.

A centered ($\Exp{X}=0$) scalar random variable $X$ is sub-Gaussian with constant $\kappa$, denoted $X \sim SG(\kappa)$, if for all $t>0$ 
it satisfies
$
    \Prob{\Abs{X} > t } \leq 2 \exp\Brack{-\frac{t^2}{\kappa^2}}
$. A random vector $X = (X_1, \ldots, X_m)$ is $\kappa$ sub-Gaussian, denoted $X\sim SG_m(\kappa)$, if for every $v \in \RR^m$ with $\norms{v}=1$ the random variable $\inner{v}{X}$ is $\kappa$ sub-Gaussian. A random vector $X$ is $\sigma$-isotropic if for every $v\in \RR^m$ with $\norms{v}=1$, $\Exp \inner{v}{X}= \sigma^2$. 

Finally, a random vector $X = (X_1, \ldots, X_m)$ is $\sigma$-isotropically $\kappa$ sub-Gaussian with independent components, denoted $X \sim ISG_m(\sigma,\kappa)$ if $X_i$ are independent, and for all $i \leq m$, $\Exp X_i = 0$, $\Exp X_i^2 = \sigma^2$ and $X_i \sim SG(\kappa)$. Clearly, if $X \sim ISG_m(\sigma,\kappa)$ then $X$ is $\sigma$-isotropic. Recall also that 
$X \sim ISG_m(\sigma,\kappa)$ implies $X \sim SG_m(\kappa)$  \citep{hdp2018}.  The noise variables we discuss in this paper are  $ISG(\kappa, \sigma)$. 

Throughout the paper, absolute constants are denoted by $c,c',c'', \ldots$. etc. Their values may change from line to line. 

\section{Overview of the approach}
\label{sec:overview}

We begin by rewriting (\ref{eq:lds_main_1})-(\ref{eq:lds_main_2}) in a vector form. 
To this end, we first encode sequences of $T$ vectors in $\RR^n$, $\Set{a_t}_{t\leq T} \subset \RR^n$,
as a vector $a \in \RR^{Tn}$, constructed by concatenation of $a_t$'s.
Next, we define the summation operator 
$S':\RR^T \rightarrow \RR^T$ which acts on any vector $(h_1,h_2,\ldots,h_T) \in \RR^T$ by 
\begin{equation}
\label{eq:Sprime_def}
S'(h_1,h_2,\ldots,h_T) = (h_1,h_1+h_2,\ldots, \sum_{i\leq T-1} h_i, \sum_{i\leq T} h_i). 
\end{equation}
Note that $S'$ is an invertible operator.
Next, 
we similarly define the summation operator $S:\RR^{Tn} \rightarrow \RR^{Tn}$, an $n$-dimensional 
extension of $S'$, which sums $n$-dimensional vectors. Formally, for 
$(h_l)_{l=1}^{Tn} \in \RR^{Tn}$, and for $1\leq j \leq n, 1 \leq t \leq T$, $(Sh)_{(t-1)\cdot n+j} = \sum_{i\leq t} h_{(i-1) \cdot n+j}$.
Observe that if the sequence of process noise terms $h_1,\ldots,h_T \in \RR^n$ is viewed as a vector 
$h \in \RR^{Tn}$, then by definition $Sh$ is the $\RR^{Tn}$ encoding of the sequence
$X_t$.

Next, given a sequence of observation vectors $u_1,\ldots,u_T\in \RR^n$, we define the observation 
operator $O_u: \RR^{Tn} \rightarrow \RR^{T}$ by $(O_u x)_t = \inner{u_t}{\Brack{x_{(t-1)\cdot n+1},\ldots,x_{(t-1)\cdot n+n}}}$. 
In words, coordinate $t$ of $O_u x$ is the inner product between $u_t$ and $t$-th part of the 
vector $x \in \RR^{Tn}$. 
Define also $Y = (Y_1,...,Y_T) \in \RR^T$ to be the concatenation of $Y_1,...,Y_T$. 
With this notation, one may equivalently rewrite the system 
(\ref{eq:lds_main_1})-(\ref{eq:lds_main_2}) as follows:
\begin{equation}
\label{eq:main_generative_equation}
Y = O_u S h + z,
\end{equation}
where $h$ and $z$ are independent zero-mean random vectors in $\RR^{Tn}$ and $\RR^T$ respectively, with independent sub Gaussian coordinates. The variance of each coordinate of $h$ is $\sigma^2$ and each coordinate of $z$ has variance $\eta^2$. 

Up to now, we have reformulated our data model as a single vector equation. Note that 
in that equation, the observations $Y$ and both operators $O_u$ and $S$ are known to us. 
Our problem may now be reformulated as follows: Given $Y\in \RR^T$, assuming 
$Y$ was generated by (\ref{eq:main_generative_equation}), provide estimates of 
$\sigma,\eta$.  

As a motivation, we first 
consider taking the expectation of the norm squared of eq. (\ref{eq:main_generative_equation}).
For any operator $A:\RR^m \rightarrow \RR^m$ and zero-mean vector $h$ with independent coordinates and coordinate variance $\sigma^2$, we have 
$\Exp{\norms{Ah}^2} = \normophs{A}^2 \sigma^2$, where $\normophs{A}$ is the Hilbert-Schmidt (or Frobenius) norm of $A$. Taking the norm and expectation of (\ref{eq:main_generative_equation}), and dividing by $T^2$, we thus obtain
\begin{equation}
\label{eq:motivation}
    \frac{\Exp{\norms{Y}^2}}{T^2} = 
    \frac{\normophs{O_uS}^2}{T^2} \sigma^2 + \frac{T}{T^2} \eta^2. 
\end{equation}
Next, note that $\normophs{O_uS}^2$ is known, and an elementary computation shows that $\frac{\normophs{O_uS}^2}{T^2}$ is of constant order (as a function of $T$; see (\ref{eq:gamma_hs_bound})), while the coefficient of $\eta^2$  is $\frac{1}{T}$. Thus, if the quantity $\frac{\norms{Y}^2}{T^2}$ were close enough to its expectation with high probability, 
we could take this quantity as a (slightly biased) estimator of $\sigma^2$. However, as it will become apparent later, the deviations of
$\frac{\norms{Y}^2}{T^2}$ around the expectation are also of constant order, and thus $\frac{\norms{Y}^2}{T^2}$ can not be used as an estimator.  The reason for these high deviations of $\frac{\norms{Y}^2}{T^2}$ is that the spectrum of $O_uS$ is extremely peaked. The highest squared singular value of $O_uS$ is of order $T^2$, the same order as sum of all of them, $\normophs{O_uS}^2$. Contrast this with the case of identity operator, $Id: \RR^{Tn} \rightarrow \RR^{Tn}$: We have 
$\Exp \norms{Id(h)}^2 = \Exp \norms{h}^2 = Tn \sigma^2$, and one can also easily show that, for instance, $Var \norms{Id(h)}^2 = Tn \sigma^2 $, and thus the deviations are of order $\sqrt{Tn}\sigma$ -- a smaller order than $\Exp \norms{Id(h)}^2$. While for the identity operator the computation is elementary, for a general operator $A$  the situation is significantly more involved, and the bounds on the deviations of $\norms{Y}^2$ will be obtained from the Hanson-Wright inequality \citep[see also \cite{rudelson2013hanson}]{hansonwright}, combined with standard norm deviation bounds for isotropic sub Gaussian vectors. 

With these observations in mind, we proceed to flatten the spectrum of $O_uS$ by taking the pseudo-inverse. Let $R:\RR^T \rightarrow \RR^{Tn}$ be the pseudo-inverse, or Moore-Penrose inverse of $O_uS$. Specifically, let 
\begin{equation}
\label{eq:O_uS_SVD_intro}
     O_uS = U \circ Diag(\gamma_1,\ldots,\gamma_T)\circ W, 
\end{equation}
be the singular value decomposition of $O_uS$, where 
$\gamma_1 \geq \gamma_2 \geq \ldots \geq \gamma_T$ are the singular values.

For the rest of the paper, we will  assume that all of the observation vectors $u_t$ are non-zero. This assumption is made solely for notational simplicity and may easily be avoided, as discussed later in 
this section. Under this assumption, since $S$ is invertible and $O_u$ has rank $T$, we have $\lambda_t>0$ for all $t\leq T$. For $i \leq T$, denote $\chi_i = \gamma_{T+1-i}^{-1}$. Then 
$\chi_i$ are the singular values of $R$, arranged in a non-increasing order, and we have by definition
\begin{equation}
\label{eq:R_SVD_intro}
     R = W^* \circ Diag(\chi_T,\chi_{T-1},\ldots,\chi_2,\chi_1)\circ U^*,
\end{equation}
where $W^{*},U^{*}$ denote the transposed matrices of $U,V$, defined in (\ref{eq:O_uS_SVD_intro}).

Similarly to Eq. (\ref{eq:motivation}), we apply $R$ to (\ref{eq:main_generative_equation}), and since $\normophs{RO_uS} = T$, by taking the expectation of the squared norm we obtain
\begin{equation}
\label{eq:first_eq_intro_exp}
\frac{\norms{RY}^2}{T} = \sigma^2 + \frac{\normophs{R}^2}{T} \eta^2   +\Brack{\frac{\norms{RY}^2}{T} - \frac{\Exp{\norms{RY}^2}}{T}}. 
\end{equation}
In this equation, the deviation term 
$\Brack{\frac{\norms{RY}^2}{T} - \frac{\Exp{\norms{RY}^2}}{T}}$ is of order  $O(\frac{1}{\sqrt{T}})$ with high probability (Theorem \ref{thm:main_thm}). Moreover, the coefficient of $\sigma^2$ is $1$, and the coefficient of $\eta^2$, which is $\frac{\normophs{R}^2}{T}$, is of order at least $\Omega(\frac{1}{\log^2 T})$ (Theorem \ref{lem:R_properties}, see Section \ref{sec:analysis} for additional details) -- much larger order than $\frac{1}{\sqrt{T}}$. Since $\norms{RY}^2$ and 
$\normophs{R}^2$ are known,
it follows that we have obtained one equation satisfied by $\sigma^2$ and $\eta^2$ up to an error of $\frac{1}{\sqrt{T}}$, where both coefficients are of order larger than the error.

Next, we would like to obtain another linear relation between $\sigma^2$, $\eta^2$. 
To this end, choose some $p = \alpha T$, where $0<\alpha<1$ is of constant order. The possible choices of $p$ are discussed later in this section. We define an operator $R':\RR^{T} \rightarrow \RR^{Tn}$ to be a version of $R$ truncated to the first $p$ singular values. If (\ref{eq:R_SVD_intro}) is the 
SVD decomposition of $R$, then 
\begin{equation*}
     R' = W^* \circ Diag(0,0,\ldots,\chi_p,\chi_{p-1},\ldots,\chi_1)\circ U^*. 
\end{equation*}
Similarly to the case for $R$, we have 
\begin{equation}
\label{eq:second_eq_intro_exp}
\frac{\norms{R'Y}^2}{p} = \sigma^2 + \frac{\normophs{R'}^2}{p} \eta^2   +\Brack{\frac{\norms{R'Y}^2}{p} - \frac{\Exp{\norms{R'Y}^2}}{p}
}.
\end{equation}
The deviations in (\ref{eq:second_eq_intro_exp}) are also described by Theorem \ref{thm:main_thm}.   Note also that since $\normophs{R'}^2$ is the sum of $p$ largest squared singular values of $R$, by definition it follows that $\frac{\normophs{R'}^2}{p} \geq \frac{\normophs{R}^2}{T}$.

Now, given two equations in two unknowns, we can solve the system to obtain the estimates 
$\widehat{\sigma^2}$ and $\widehat{\eta^2}$. The full procedure is summarized in Algorithm \ref{alg:ve}, and the bounds implied by Theorem \ref{thm:main_thm} on the estimators
$\widehat{\sigma^2}$ and $\widehat{\eta^2}$ 
are given in Corollary \ref{cor:estimator_bounds}. We first state these results, and then discuss in detail the various parameters appearing in the bounds. 

\begin{algorithm}[tb]
   \caption{Spectrum Thresholding Variance Estimator (STVE)}
   \label{alg:ve}
\begin{algorithmic}[1]
   \STATE {\bfseries Input:} Observations $Y_t$, observation vectors $u_t$, with $t\leq T$, and $p=\alpha T$. \\ 
   \STATE Compute the SVD of $O_uS$, 
    \begin{equation*} 
       O_uS = U \circ Diag(\gamma_1,\ldots,\gamma_T)\circ W, 
    \end{equation*}   
   where $\gamma_1 \geq \gamma_2 \geq \ldots \geq \gamma_T > 0$. Denote $\chi_i = \gamma_{T+1-i}^{-1}$ for $1\leq i \leq T$.\\
	\STATE Construct the operators 
    \begin{equation*} 
       R = W^{*} \circ Diag(\chi_T,\ldots,\chi_1)\circ U^{*} 
    \end{equation*}      
    and 
    \begin{equation*} 
       R' = W^{*} \circ Diag(0,0,\ldots,0,\chi_p,\ldots,\chi_1)\circ U^{*}   
    \end{equation*} 
    
   \STATE Produce the estimates: 
      \begin{eqnarray*}
          \widehat{\eta^2} &=& \Brack{ 
          							 \frac{\norms{R'Y}^2}{p} - 
          							 \frac{\norms{RY}^2}{T} 
          						} \Big/ \Brack{
          						     \frac{\normophs{R'}^2}{p} - 
          						     \frac{\normophs{R}^2}{T} 
          						}   \\
          \widehat{\sigma^2} &=& \frac{\norms{RY}^2}{T} - \frac{\normophs{R}^2}{T}\widehat{\eta^2}.      
      \end{eqnarray*}
\end{algorithmic}
\end{algorithm}

\begin{thm}  
\label{thm:main_thm}
Consider a random vector $Y \in \RR^T$ of the form $Y = O_u Sh + z$
where $h \sim ISG_{Tn}(\sigma,\kappa)$ and $z \sim ISG_{T}(\eta,\kappa)$.  
Set $\norms{u_{min}} = \min_t \norms{u_t}$.
Then  for any $0<\delta<1$,
\begin{flalign}
\label{eq:main_thm_dev_for_R_statement}
\Prob{\Abs{\frac{\norms{RY}^2}{T} - \Brack{\sigma^2 + \frac{\normophs{R}^2}{T}\eta^2}} \geq 
c \frac{B}{\sqrt{T}}
} \leq 4 \delta,
\\
\label{eq:main_thm_dev_for_R_tag_statement}
\Prob{\Abs{\frac{\norms{R'Y}^2}{p} - \Brack{\sigma^2 + \frac{\normophs{R'}^2}{p}\eta^2}} \geq 
c \frac{B}{\sqrt{p}}
} \leq 4 \delta
\end{flalign}
where $B$ is given by 
\begin{equation}
\label{eq:B_delta_def}
    B = \Brack{1+\kappa^2}\Brack{1+\norms{u_{min}}^{-2}}\log \frac{1}{\delta}.
\end{equation}
\end{thm}

The bounds on the estimators of Algorithm \ref{alg:ve} are given in the following Corollary. As discussed below, in addition to Theorem \ref{thm:main_thm}, the key to the derivation of this Corollary are the estimates of the spectrum of $R$, given in  Theorem \ref{lem:R_properties},  Section \ref{sec:analysis}.
\begin{cor} 
\label{cor:estimator_bounds}
Let $\widehat{\sigma^2},\widehat{\eta^2}$ be the estimators of $\sigma^2,\eta^2$ obtained form Algorithm \ref{alg:ve} with $p\geq \frac{1}{4}T$. Set $\norms{u_{max}} = \max_t \norms{u_t}$.
Then for any $0<\delta<1$,
with probability at least $1- 8 \delta$, 
\begin{align}
\label{eq:corr_est_bound_sigma}
\Abs{\widehat{\sigma^2} - \sigma^2}  &\leq 
c\frac{B}{\sqrt{T}}
 \Brack{1- \frac{p\normophs{R}^2}{T\normophs{R'}^2}}^{-1}, 
\\
\label{eq:corr_est_bound_eta}
\Abs{\widehat{\eta^2} - \eta^2} &\leq 
c\frac{B}{\sqrt{T}}
 \Brack{1- \frac{p\normophs{R}^2}{T\normophs{R'}^2}}^{-1} n^2 \norms{u_{max}}^2 \log^2 T,
\end{align}
with $B$ given by \eqref{eq:B_delta_def}.
\end{cor}

We first discuss the assumption $\norms{u_{min}} > 0$. This assumption is made solely for notational convenience, as detailed below. To begin, note that some 
form of lower bound on the norms of the observation vectors $u_t$ \textit{must} appear in the bounds.
 This is simply because if one had $u_t = 0$ for all $T$, then clearly no estimate of $\sigma$ 
would have been possible. On the other hand, our use of the smallest value $\norms{u_{\min}}$ may 
seem restrictive at first. We note however, 
that instead of considering the observation 
operator $O_u : \RR^{Tn} \rightarrow \RR^T$, 
one may consider the operator 
$O_{\bar{u}} : \RR^{Tn} \rightarrow \RR^{\bar{T}}$ for any subsequence 
$\Set{\bar{u}_{\bar{t}}}_{\bar{t}=1}^{\bar{T}}$. The observation vector $Y$ would be correspondingly restricted to the subsequence of indices. This allows us to treat missing values and to exclude any outlier $u_t$ with small norms. All the arguments in Theorems \ref{thm:main_thm} and \ref{lem:R_properties} hold for this modified $O_{\bar{u}}$ without change. The only price that will be paid is that $T$ will be replaced by $\bar{T}$ in the bounds. Moreover, we note that typically we have $\norms{u_t}\geq 1$  by construction, see for instance Section \ref{sec:electricity_experiments}. Additional discussion of missing values may be found in Supplementary Material Section \ref{sec:missing_values}.

Next, up to this point, we have obtained two equations, (\ref{eq:first_eq_intro_exp})-(\ref{eq:second_eq_intro_exp}), in two unknowns, $\sigma^2,\eta^2$. Note that in order to be able to obtain $\eta^2$ from these equations, at least one of the coefficients of $\eta^2$, either $\frac{\normophs{R}^2}{T}$ or $\frac{\normophs{R'}^2}{p}$ must be of larger order than $\frac{1}{\sqrt{T}}$, the order of deviations. Providing lower bounds on these quantities is one of the main technical contributions of this work.  This analysis uses the connection between the operator $S$ and the Laplacian on the line, and resolves the issue of translating spectrum estimates for the Laplacian into the spectral estimates for $R$. We note that there are no standard tools to study the spectrum of $R$, and our approach proceeds indirectly via the analysis of the nuclear norm of $O_uS$.  These results are stated in Theorem \ref{lem:R_properties}. In particular, we show that $\frac{\normophs{R}^2}{T}$ is $\Omega(\frac{1}{\log^2 T})$, which is the source of the log factor in (\ref{eq:corr_est_bound_eta}).

\begin{figure}
\centering
\subcaptionbox{
\label{fig:electric_spectrum}
Spectrum ($\chi_i^2$) of $R$ for electricity data (blue). $\normophs{R'(p)}^2/p$ as a function of $p$ (orange). The mean $\normophs{R}^2/T$ (green).  }{\includegraphics[width=0.50\textwidth]{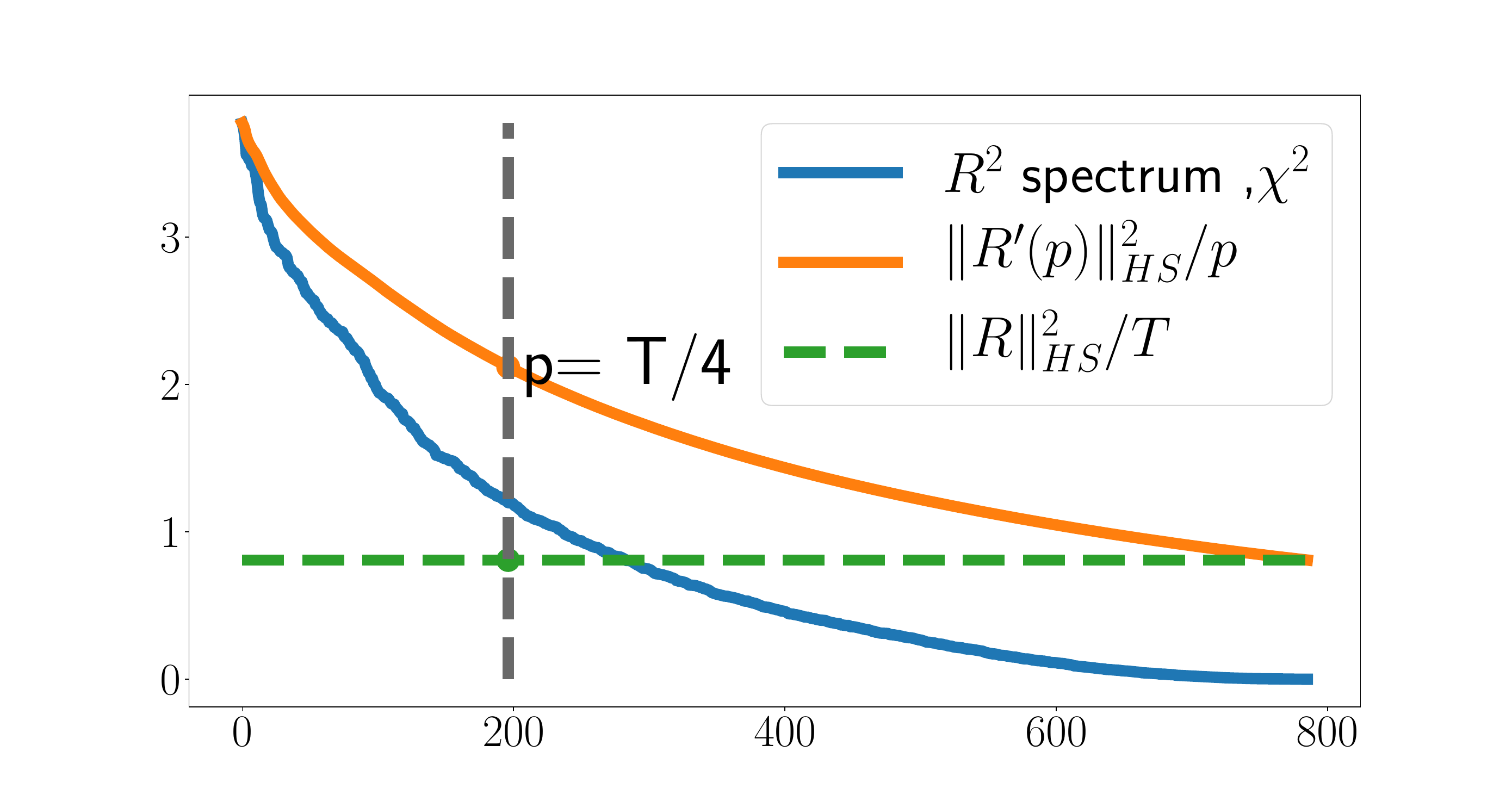}}%
\hfill 
\subcaptionbox{
\label{fig:spec_T500} 
Similar figure for synthetic data, $T=500$.  }{\includegraphics[width=0.50\textwidth]{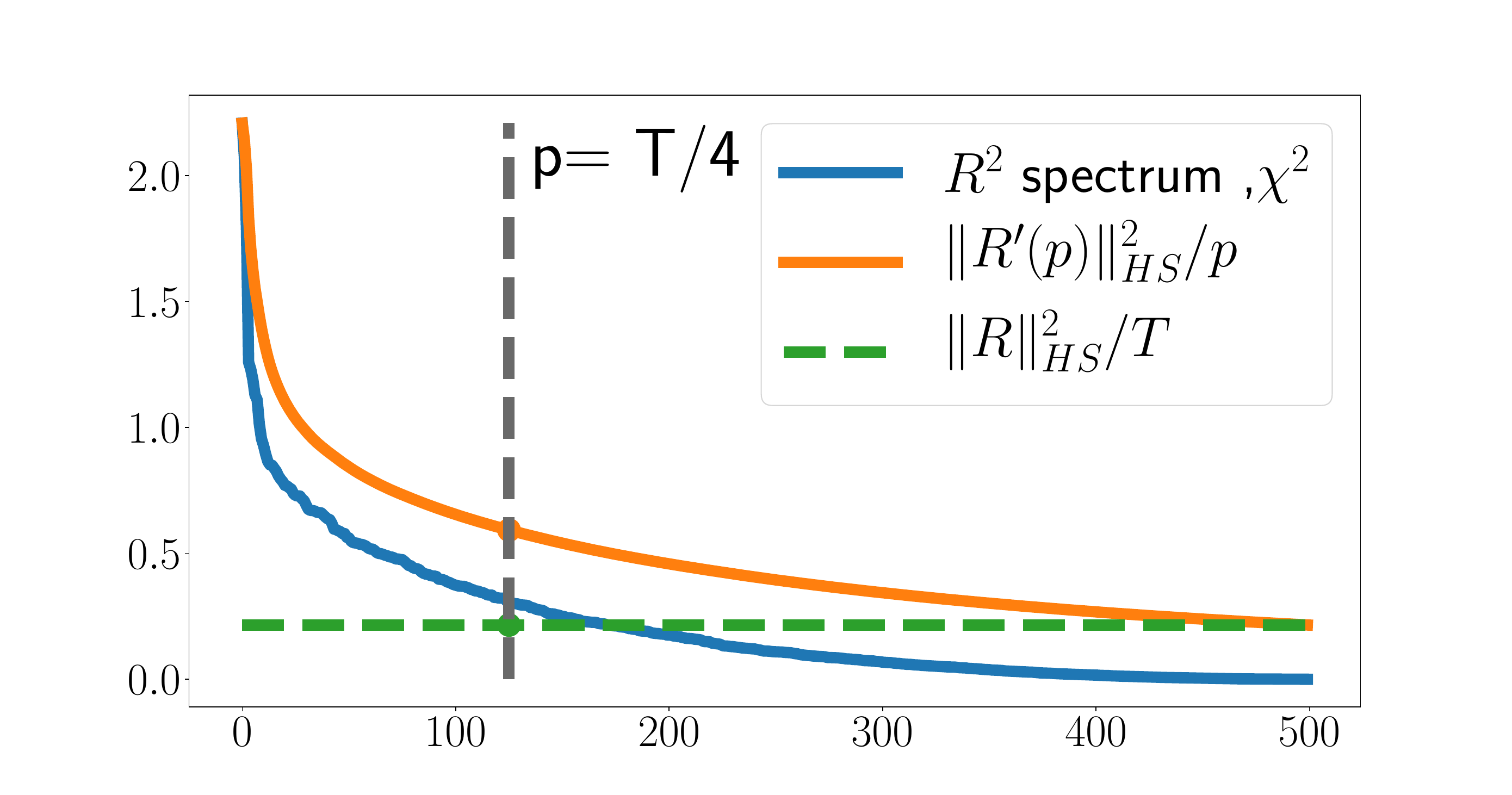}}%
\caption{Spectra of $R$}
\end{figure}

Finally, in order to solve the equations 
(\ref{eq:first_eq_intro_exp})-(\ref{eq:second_eq_intro_exp}), not only the equations must have large enough coefficients, but the equations must be \emph{different}. This is reflected by the term $\Brack{1- \frac{p\normophs{R}^2}{T\normophs{R'}^2}}^{-1}$ in (\ref{eq:corr_est_bound_sigma}), (\ref{eq:corr_est_bound_eta}). Equivalently, 
while $\frac{\normophs{R'}^2/p}{\normophs{R}^2/T} \geq 1$ by definition, we would like to have 

\begin{equation}
\label{eq:stable_bounds_desired_cond}
\frac{\normophs{R'}^2/p}{\normophs{R}^2/T} \geq 1+const 
\end{equation}
for the bounds (\ref{eq:corr_est_bound_sigma}), (\ref{eq:corr_est_bound_eta}) to be stable. Note that since both $\normophs{R'}^2$ and  $\normophs{R}^2$ are computed in Algorithm \ref{alg:ve}, the condition (\ref{eq:stable_bounds_desired_cond}) can simply be verified before the estimators $\widehat{\sigma^2}, \widehat{\eta^2}$ are returned. 

It is worth emphasizing that for simple choices of $p$, say $p=\frac{1}{4}T$, the condition (\ref{eq:stable_bounds_desired_cond}) does hold in practice. Note that, for any $p$, we can have $\frac{\normophs{R'}^2/p}{\normophs{R}^2/T} = 1$ only if the spectrum of $R$ is \emph{constant}. Thus 
(\ref{eq:stable_bounds_desired_cond}) amounts to stating that the spectrum of $R$ exhibits some decay.
As we show in experiments below, the spectrum (squared) of 
 $R$, for $u_t$ derived from daily temperature features, or for random Gaussian $u_t$, indeed decays. See Section \ref{sec:experiments},  Figures \ref{fig:electric_spectrum} and \ref{fig:spec_T500}.   In particular, in both cases (\ref{eq:stable_bounds_desired_cond}) holds with $const > 1$. Additional  bounds on the quantity $\Brack{1- \frac{p\normophs{R}^2}{T\normophs{R'}^2}}^{-1}$ under various assumptions on the sequence $u_t$ are given in Section \ref{sec:gap_analysis} of the Supplementary Material.

\section{Properties of $O_uS$ and $R$}
\label{sec:analysis}
As discussed in Section \ref{sec:overview} (see the discussion following eq. \eqref{eq:first_eq_intro_exp}), one of the crucial points enabling Algorithm \ref{alg:ve} and its analysis is the fact that the quantity $\frac{\normophs{R}^2}{T}$ is bounded below by an expression that is of  much higher order than the noise magnitude $\frac{1}{\sqrt{T}}$.

In this section we provide the formal statement of this and other associated results, and discuss the related arguments. First, we obtain the following bound on the spectrum of $O_uS$ (Lemma \ref{lem:OS_properties}, Supplementary Material Section \ref{sec:supp_os_properties}). Recall that the nuclear norm was defined in Section \ref{sec:notation},  and that for a sequence $u_t$ we set $\Abs{u_{max}} = \max_t \Abs{u_t}$ and $\Abs{u_{min}} = \min_t \Abs{u_t}$. Then:
\begin{align}
\label{eq:gamma_nuclear_bound_demo}
 \normopnuc{O_uS} = \sum_{t\leq T} \lambda_t(O_u S) \leq 4n \norms{u_{max}} T \log T .
\end{align}
The proof of this bound exploits the connection between $S$ and the Laplacian on the line. In particular, we use the fact that the eigenvalues of the Laplacian are known precisely, satisfying $\lambda_l = 2\sin \Brack{\frac{\pi (T-l) }{2T}}$. Next, we state the lower (and upper) bounds for $R$. 
\begin{thm} 
\label{lem:R_properties}
Let $R:\RR^T \rightarrow \RR^{Tn}$ be the pseudoinverse of $O_uS$. 
Then 
\begin{align}
c \frac{1}{ n \norms{u_{max}} \log T} &\leq \normop{R} \leq 2 \norms{u_{min}}^{-1}, \\  
 c \frac{1}{ n^2 \norms{u_{max}}^2 \log^2 T}  T &\leq 
 \normophs{R}^2 \leq 4 \norms{u_{min}}^{-2} T, \\ 
c \frac{1}{ n^4 \norms{u_{max}}^4 \log^4 T}  T &\leq 
\normophs{R^*R}^2 \leq 16 \norms{u_{min}}^{-4} T.   
\end{align}
\end{thm}
Due to the complicated structure of $R$ as a pseudo-inverse of a composition of operators, there are no direct ways to control individual eigenvalues of $R$.
Thus the main technical issue resolved in Theorem \ref{lem:R_properties} is nevertheless obtaining lower bounds on $\normophs{R}^2$. Our approach is rather indirect, and we obtain these bounds from the nuclear norm bound (\ref{eq:gamma_nuclear_bound_demo}) via a Markov type inequality on the eigenvalues. 

\section{Experiments}
\label{sec:experiments}

\subsection{Synthetic Data}
\label{sec:synthetic}
In this section the performance of STVE is evaluated on synthetic data. The data was generated by the LDS (\ref{eq:lds_main_1})-(\ref{eq:lds_main_2}), 
using Gaussian noises with  $\sigma^2=0.5, \eta^2=2$. The input dimension was $n=5$, and the input sequence $u_t$ sampled from the Gaussian $N(0,I_n)$.

We run the STVE algorithm for different values of $T$, where for each $T$ we sampled the data $150$ times. 
Figure \ref{fig:synthetic_errors} shows the average (over 150 runs) estimation error for both process and observation noise variances for various values of $T$. As expected from the bounds in Corollary \ref{cor:estimator_bounds}, it may be observed in Figure \ref{fig:synthetic_errors} that
the estimation errors decay roughly at the rate of $const/\sqrt{T}$. A typical spectrum of $R$ is shown in Figure \ref{fig:spec_T500}. For larger $T$, the spectra also exhibits similar decay.

\subsection{Temperatures and Electricity Consumption}
\label{sec:electricity_experiments}

\begin{figure}
\centering
\subcaptionbox{
    STVE variance estimation errors vs $T$, and the $T^{-\half}$ decay.
    \label{fig:synthetic_errors}
}{\includegraphics[width=0.32\textwidth]{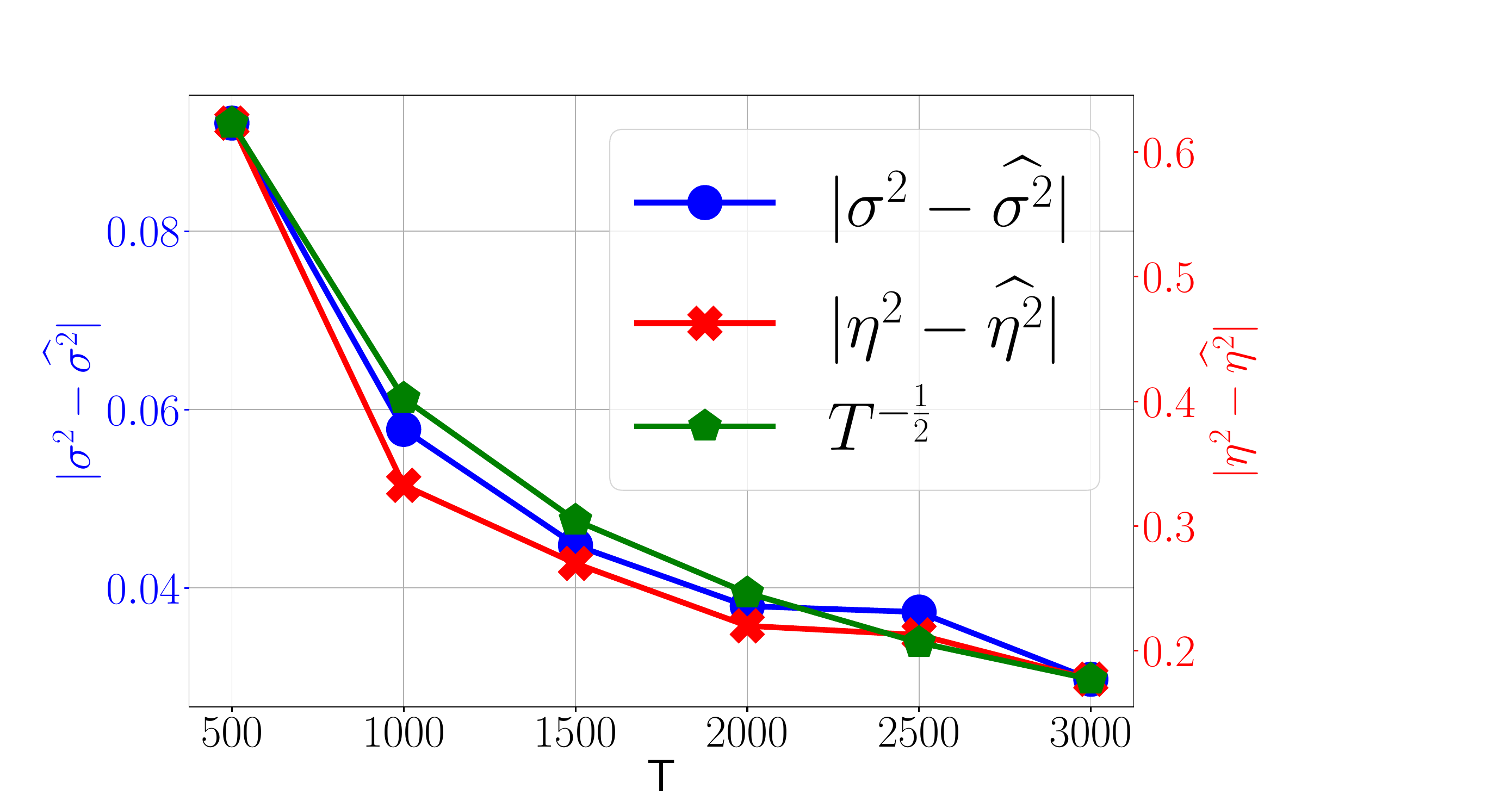}
}
\hfill 
\subcaptionbox{
\label{fig:electic_staionary_first_second}
Load (y-axis) against Temperature (x-axis), both axis normalized. The full data (blue points), regression learned on the first half (orange), regression learned on the second half (green).}{\includegraphics[width=0.32\textwidth]{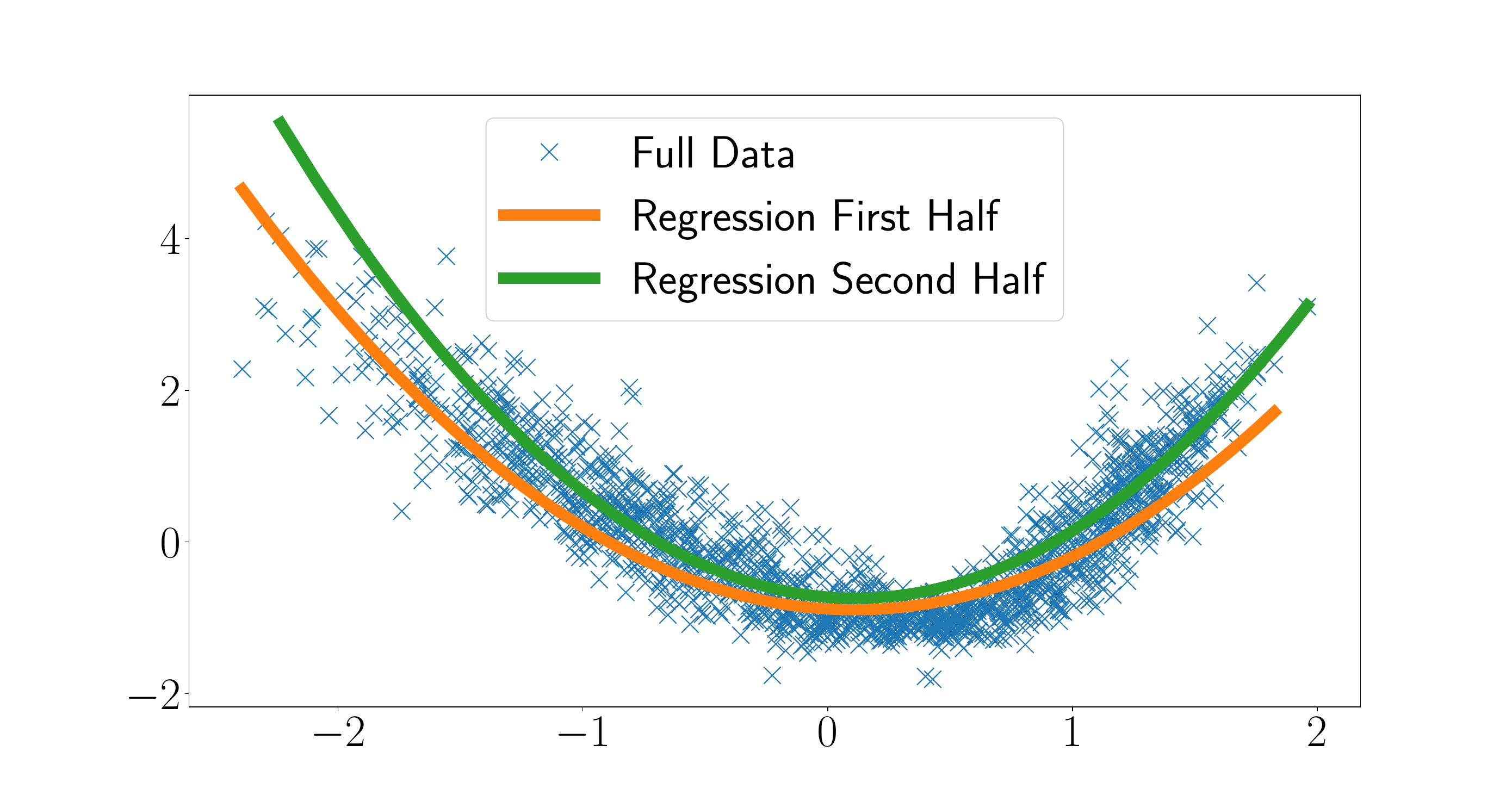}
}
\hfill 
\subcaptionbox{
\label{fig:electric_zone_0_err}
Smoothed prediction errors. Stationary Regression trained on first half (blue), MLE (orange), STVE (green), OG (red).
}{\includegraphics[width=0.32\textwidth]{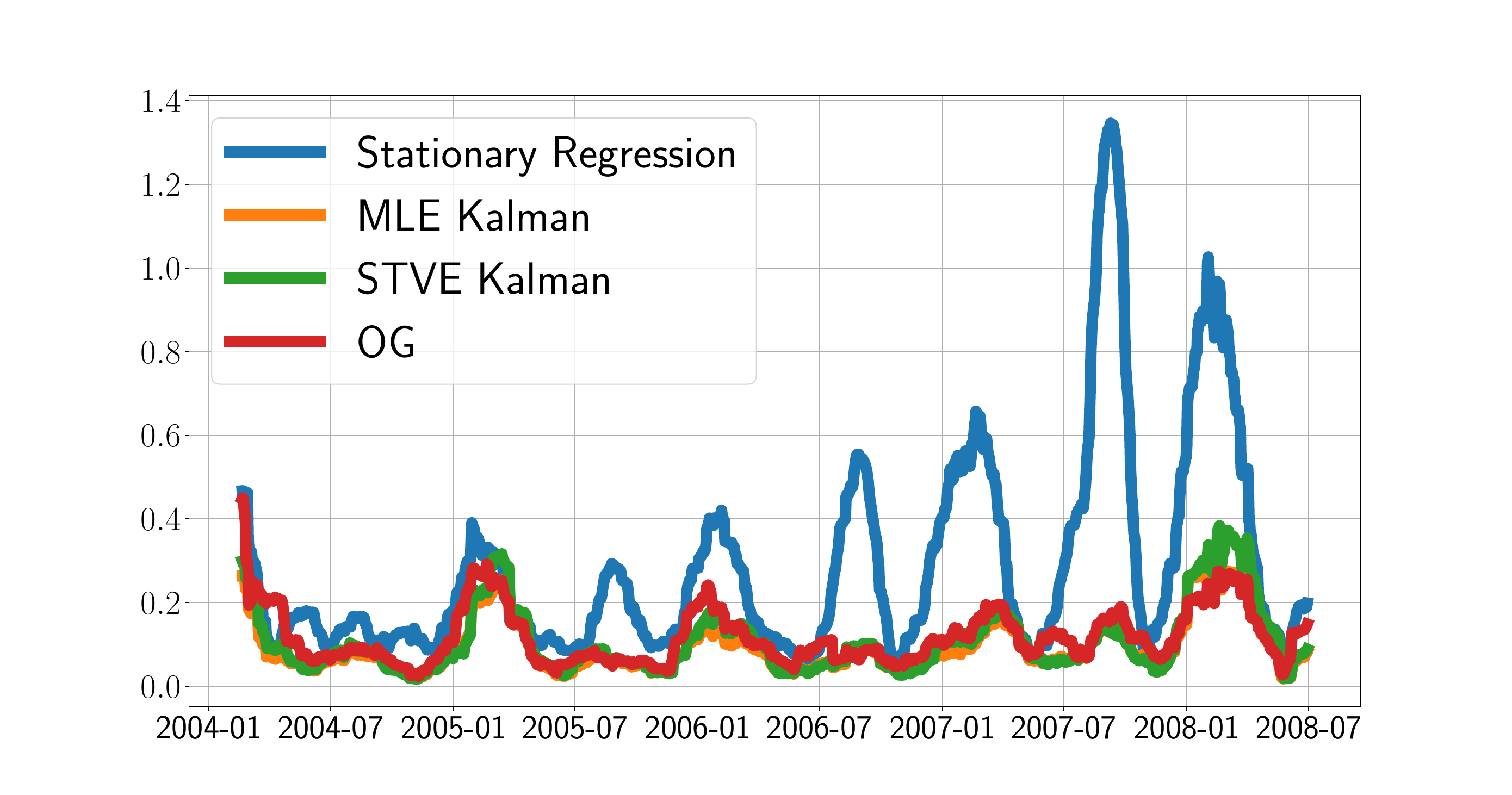}
}
\caption{Evaluation}
\end{figure}

In this section we examine the relation between daily 
temperatures and electricity consumption in the data from \cite{electric_data_paper} (see also \cite{electric_data_blog}). The following forecasting methods are compared: a stationary regression, an online gradient, and 
a Kalman filter for a dynamic regression, with parameters learned via MLE or STVE. 
We find that the Kalman filter methods provide the best performance, with no 
significant difference between STVE and MLE derived systems. 

The data consists of total daily electricity consumption (load) $y_t$, 
and the average daily temperature, $v_t$,
for a certain region, for the period Jan-$2004$ to Jun-$2008$. 
Full details on the preprocessing of the data, as well as additional details on the experiments,  are given in  Supplementary Material Section \ref{sec:electicity_additional_details}. Here we note that the data contains missing load values, for 9 non-consecutive weeks (out of about 234 weeks total).  All methods discussed here, including STVE, can naturally incorporate missing values, as discussed in Supplementary Material Section \ref{sec:missing_values}.

An elementary inspection of the data reveals that the 
load may be reasonably approximated as a quadratic 
function of the temperature,
$y_t = x_{t,1}\cdot 1 +  x_{t,2} \cdot v_t + x_{t,3} \cdot v_t^2$,
where $u_t = (1,v_t,v_t^2)$ is the observation vector 
(features), and $x_t = (x_{t,1}, x_{t,2},  x_{t,3})$ is 
the possibly time varying regression vector. 
This is shown in Figure 
\ref{fig:electic_staionary_first_second}, where we fit 
a stationary (time invariant) regression of the above form, 
using either only the 
first or only the second half of the data. We note 
that these regressions \emph{differ} -- the regression vector changes with time. It is therefore 
of interest to track it via online regression.

We use the first half of the data (train set) to learn the parameters $\sigma,\eta$ of the online 
regression (\ref{eq:lds_main_1})-(\ref{eq:lds_main_2}) via MLE optimization and using STVE. We also use the train set to find the 
optimal learning rate $\alpha$ for the OG forecaster described by the update equation (\ref{eq:online_grad_update}). This learning 
rate is chosen as the rate that yields smallest least squares forecast error on the train set. In addition, we learn a time independent, stationary regression on the first half of the data. 

We then employ the learned parameters to make predictions of the load given the temperature, by all four methods. 
The predictions for the system (\ref{eq:lds_main_1})-(\ref{eq:lds_main_2}) are made with a Kalman filter (at time $t$, we use the 
filtered state estimate $\tilde{x}_t$, which depends only on $y_1,\ldots,y_t$ and $u_1,\ldots,u_t$, and make the prediction 
$\tilde{y}_{t+1} = \inner{\tilde{x}_t}{u_{t+1}}$).

Daily squared prediction errors (that is, $\Brack{y_t - \tilde{y}_t}^2$) are shown in Figure \ref{fig:electric_zone_0_err} (smoothed with a moving average of 50 days). We see that the adaptive models (MLE, STVE, OG) outperform the stationary regression already on the train set (first half of the data), and that the difference in performance becomes dramatic on the second half (test). It is  also interesting to note that the performance of the Kalman filter based methods (MLE, STVE) is practically identical, but both are somewhat better than the simpler OG approach. 

We also note that by construction, we have $\Abs{u_{min}} \geq 1$ in this experiment, due to the constant $1$ coordinate, and 
also $\Abs{u_{max}} \leq 5$, due to the normalization. Since these operations are typical for any regression problem, 
we conclude that the direct influence of $\Abs{u_{min}}$ and  $\Abs{u_{max}}$ on the bounds in Theorem \ref{thm:main_thm} and Corollary \ref{cor:estimator_bounds} will not 
usually be significant.

\section{Conclusion And Future Work}
\label{sec:conclusion}
In this work we introduced the STVE algorithm for estimating the variance parameters of LDSs of type (\ref{eq:lds_main_1})-(\ref{eq:lds_main_2}), and obtained the first sample complexity guarantees for such estimators. We have also shown how the shape of the spectrum of $R$ can be exploited to obtain the estimators and the related bounds, thus providing the first explicit geometric parameter of the data that affects the bounds.

As discussed in Section \ref{sec:introduction}   and demonstrated in Section \ref{sec:experiments}, the system (\ref{eq:lds_main_1})-(\ref{eq:lds_main_2}) is of independent interest in applications. However, we also believe that the analysis presented here is an important first step towards a finite time data-dependent quantitative understanding of  general LDSs. and perhaps even non-linear dynamical systems.

\begin{ack}
This research was supported by the ISRAEL SCIENCE FOUNDATION
(grant No. 2199/20).
\end{ack}

\bibliographystyle{apalike}
\bibliography{oe}

\begin{thebibliography}{}

\bibitem[Adamczak, 2015]{adamczak}
Adamczak, R. (2015).
\newblock A note on the hanson-wright inequality for random vectors with
  dependencies.
\newblock {\em Electronic Communications in Probability}, 20.

\bibitem[Anava et~al., 2013]{anava13}
Anava, O., Hazan, E., Mannor, S., and Shamir, O. (2013).
\newblock Online learning for time series prediction.
\newblock In {\em {COLT} 2013 - The 26th Annual Conference on Learning Theory,
  June 12-14, 2013, Princeton University, NJ, {USA}}.

\bibitem[Anderson and Moore, 1979]{anderson1979}
Anderson, B. and Moore, J. (1979).
\newblock {\em Optimal Filtering}.
\newblock Prentice Hall.

\bibitem[Belanger, 1974]{BELANGER1974267}
Belanger, P.~R. (1974).
\newblock Estimation of noise covariance matrices for a linear time-varying
  stochastic process.
\newblock {\em Automatica}, 10(3).

\bibitem[Bhatia, 1997]{bhatia97matrix}
Bhatia, R. (1997).
\newblock {\em Matrix Analysis}.
\newblock Graduate Texts in Mathematics. Springer New York.

\bibitem[Campi and Weyer, 2005]{campi2005guaranteed}
Campi, M.~C. and Weyer, E. (2005).
\newblock Guaranteed non-asymptotic confidence regions in system
  identification.
\newblock {\em Automatica}, 41(10):1751--1764.

\bibitem[Chui and Chen, 2017]{chui2017kalman}
Chui, C. and Chen, G. (2017).
\newblock {\em Kalman Filtering: with Real-Time Applications}.
\newblock Springer International Publishing.

\bibitem[Dunik et~al., 2018]{DUNIK201816}
Dunik, J., Kost, O., and Straka, O. (2018).
\newblock Design of measurement difference autocovariance method for estimation
  of process and measurement noise covariances.
\newblock {\em Automatica}, 90.

\bibitem[Gohberg and Krein, 1969]{gohberg1969introduction}
Gohberg, I. and Krein, M. (1969).
\newblock {\em Introduction to the Theory of Linear Nonselfadjoint Operators}.
\newblock Translations of mathematical monographs. American Mathematical
  Society.

\bibitem[Hamilton, 1994]{hamilton1994time}
Hamilton, J. (1994).
\newblock {\em Time Series Analysis}.
\newblock Princeton University Press.

\bibitem[Hanson and Wright, 1971]{hansonwright}
Hanson, D.~L. and Wright, E.~T. (1971).
\newblock A bound on tail probabilities for quadratic forms in independent
  random variables.
\newblock {\em Ann. Math. Statist.}, 42.

\bibitem[Hazan, 2016]{Hazanbook2016}
Hazan, E. (2016).
\newblock Introduction to online convex optimization.
\newblock {\em Found. Trends Optim.}

\bibitem[Hazan et~al., 2017]{hazan2017online}
Hazan, E., Singh, K., and Zhang, C. (2017).
\newblock Online learning of linear dynamical systems.
\newblock In {\em Advances in Neural Information Processing Systems}, pages
  6686--6696.

\bibitem[Hong, 2016]{electric_data_blog}
Hong, T. (2016).
\newblock Tao hongs{'}s blog.
\newblock
  \url{http://blog.drhongtao.com/2016/07/gefcom2012-load-forecasting-data.html}.
\newblock Accessed: 1/8/2019.

\bibitem[Hong et~al., 2014]{electric_data_paper}
Hong, T., Pinson, P., and Fan, S. (2014).
\newblock Global energy forecasting competition 2012.
\newblock {\em International Journal of Forecasting}, 30:357--363.

\bibitem[Kozdoba et~al., 2019]{kalman_filter_decay}
Kozdoba, M., Marecek, J., Tchrakian, T.~T., and Mannor, S. (2019).
\newblock On-line learning of linear dynamical systems: Exponential forgetting
  in {K}alman filters.
\newblock {\em AAAI}.

\bibitem[{Mehra}, 1970]{MEHRA70}
{Mehra}, R. (1970).
\newblock On the identification of variances and adaptive kalman filtering.
\newblock {\em IEEE Transactions on Automatic Control}, 15(2).

\bibitem[Mitchell and Griffiths, 1980]{mitchell1980finite}
Mitchell, A. and Griffiths, D. (1980).
\newblock {\em The finite difference method in partial differential equations}.
\newblock Wiley-Interscience publication. Wiley.

\bibitem[Petris, 2010]{rdlm}
Petris, G. (2010).
\newblock An {R} package for dynamic linear models.
\newblock {\em Journal of Statistical Software}.

\bibitem[Qin, 2006]{qin2006overview}
Qin, S.~J. (2006).
\newblock An overview of subspace identification.
\newblock {\em Computers \& chemical engineering}, 30(10-12):1502--1513.

\bibitem[Rudelson et~al., 2013]{rudelson2013hanson}
Rudelson, M., Vershynin, R., et~al. (2013).
\newblock Hanson-wright inequality and sub-gaussian concentration.
\newblock {\em Electronic Communications in Probability}, 18.

\bibitem[Shumway and Stoffer, 2011]{shumwaystoffer}
Shumway, R. and Stoffer, D. (2011).
\newblock {\em Time Series Analysis and Its Applications (3rd ed.)}.

\bibitem[Tsiamis and Pappas, 2019]{tsiamis2019finite}
Tsiamis, A. and Pappas, G.~J. (2019).
\newblock Finite sample analysis of stochastic system identification.
\newblock In {\em 2019 IEEE 58th Conference on Decision and Control (CDC)},
  pages 3648--3654. IEEE.

\bibitem[van Overschee and de~Moor, 1996]{van1996subspace}
van Overschee, P. and de~Moor, L. (1996).
\newblock {\em Subspace identification for linear systems: theory,
  implementation, applications}.
\newblock Kluwer Academic Publishers.

\bibitem[Vershynin, 2018]{hdp2018}
Vershynin, R. (2018).
\newblock {\em High-Dimensional Probability: An Introduction with Applications
  in Data Science}.
\newblock Cambridge Series in Statistical and Probabilistic Mathematics.
  Cambridge University Press.

\bibitem[Vidyasagar and Karandikar, 2006]{vidyasagar2006learning}
Vidyasagar, M. and Karandikar, R.~L. (2006).
\newblock A learning theory approach to system identification and stochastic
  adaptive control.
\newblock {\em Probabilistic and randomized methods for design under
  uncertainty}, pages 265--302.

\bibitem[Wang et~al., 2017]{WANG201712}
Wang, H., Deng, Z., Feng, B., Ma, H., and Xia, Y. (2017).
\newblock An adaptive kalman filter estimating process noise covariance.
\newblock {\em Neurocomputing}, 223.

\bibitem[West and Harrison, 1997]{WestHarrison}
West, M. and Harrison, J. (1997).
\newblock {\em Bayesian Forecasting and Dynamic Models (2nd ed.)}.
\newblock Springer-Verlag.

\bibitem[Zinkevich, 2003]{Zinkevich2003}
Zinkevich, M. (2003).
\newblock Online convex programming and generalized infinitesimal gradient
  ascent.
\newblock ICML.

\end{thebibliography}

\section*{Checklist}

\begin{enumerate}

\item For all authors...
\begin{enumerate}
  \item Do the main claims made in the abstract and introduction accurately reflect the paper's contributions and scope?
    \answerYes{}
  \item Did you describe the limitations of your work?
    \answerYes{}
  \item Did you discuss any potential negative societal impacts of your work?
    \answerNA{ The main focus of this work is a theoretical analysis. }
  \item Have you read the ethics review guidelines and ensured that your paper conforms to them?
    \answerYes{}
\end{enumerate}

\item If you are including theoretical results...
\begin{enumerate}
  \item Did you state the full set of assumptions of all theoretical results?
    \answerYes{}
        \item Did you include complete proofs of all theoretical results?
    \answerYes{ Thus Supplementary Materail contains all the proofs. Proofs outlines are given in the main body of the paper. }
\end{enumerate}

\item If you ran experiments...
\begin{enumerate}
  \item Did you include the code, data, and instructions needed to reproduce the main experimental results (either in the supplemental material or as a URL)?
 \answerYes{ The data is publically available, see the references in Section \ref{sec:experiments}. The short code is not provided at the moment, but can be fully derived from Algorithm \ref{alg:ve}. }
  \item Did you specify all the training details (e.g., data splits, hyperparameters, how they were chosen)?
    \answerYes{ See Sections \ref{sec:experiments} and \ref{sec:electicity_additional_details}.}
    \item Did you report error bars (e.g., with respect to the random seed after running experiments multiple times)?
    \answerYes{}
    \item Did you include the total amount of compute and the type of resources used (e.g., type of GPUs, internal cluster, or cloud provider)?
    \answerYes{}
\end{enumerate}

\item If you are using existing assets (e.g., code, data, models) or curating/releasing new assets...
\begin{enumerate}
  \item If your work uses existing assets, did you cite the creators?
    \answerYes{}{}
  \item Did you mention the license of the assets?
    \answerNA{The link to the data author's documentation is provided.}
  \item Did you include any new assets either in the supplemental material or as a URL?
    \answerNo{}
  \item Did you discuss whether and how consent was obtained from people whose data you're using/curating?
    \answerNA{}
  \item Did you discuss whether the data you are using/curating contains personally identifiable information or offensive content?
    \answerNA{}
\end{enumerate}

\item If you used crowdsourcing or conducted research with human subjects...
\begin{enumerate}
  \item Did you include the full text of instructions given to participants and screenshots, if applicable?
    \answerNA{}
  \item Did you describe any potential participant risks, with links to Institutional Review Board (IRB) approvals, if applicable?
    \answerNA{}
  \item Did you include the estimated hourly wage paid to participants and the total amount spent on participant compensation?
    \answerNA{}
\end{enumerate}

\end{enumerate}

\appendix

\newpage
\onecolumn

{\LARGE \centering{\textbf{Finite Sample Analysis Of Dynamic Regression Parameter Learning - Supplementary Material}}} 
\newline

\begin{figure}
\centering
\includegraphics[width=.5\textwidth, height=5cm]{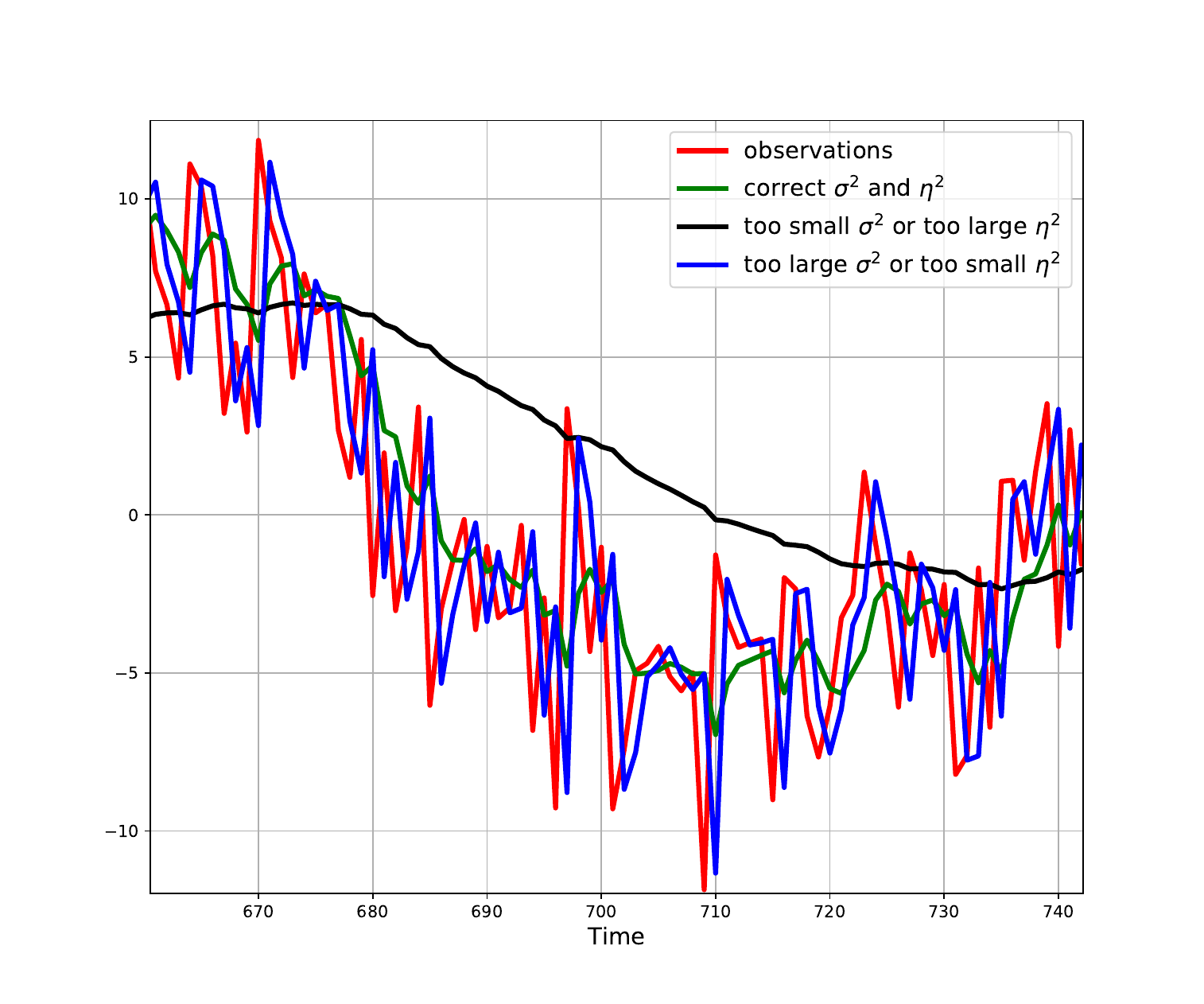}
\caption{Lag and Overfitting to the most recent observation,  for various variance values. See the discussion in Sections \ref{sec:introduction} and  \ref{sec:appendix_outline}.}
\label{fig:params}
\end{figure}

\section{Outline}
\label{sec:appendix_outline}

This Supplementary Material  is organized as follows: 
 The proofs of the results of Section \ref{sec:analysis}, including Theorem \ref{lem:R_properties}, as well as proofs of Theorems \ref{thm:main_thm}, and Corollary \ref{cor:estimator_bounds}, are given in Sections \ref{sec:Sinv_proof} to \ref{sec:coroll_proof}.  In Section \ref{gap_bound_general_proof} we state and prove additional bounds on the quantity  $\Brack{1- \frac{p\normophs{R}^2}{T\normophs{R'}^2}}^{-1}$ that appears in Corollary \ref{cor:estimator_bounds}. 
Section \ref{sec:missing_values} contains the discussion of missing values in STVE. Additional details on the electricity consumption experiment are given in Section \ref{sec:electicity_additional_details}.

Finally, we describe Figure \ref{fig:params}. The data (observations, red) was generated from the system \basesystem with one dimensional state ($n=1$), and we set $u_t=1$.  The states produced by the Kalman filter with ground truth values of $\sigma,\eta$ are shown in green, while states obtained from  other choices of parameters are shown in black and blue.

\section{Bounds on $O_uS$, Lemma \ref{lem:OS_properties}}
\label{sec:supp_os_properties}
In this Section we prove the following bounds on the spectrum of $O_uS$. See Section \ref{sec:analysis} for a discussion of these bounds. 
\begin{lem}
\label{lem:OS_properties}
The singular values of $O_u S$ satisfy  the following:
\begin{align}
\label{eq:gamma_upper_lower}
\lambda_1(O_u S) \leq  \norms{u_{max}} \cdot T  \mbox{ and } \lambda_T(O_u S) \geq \half \norms{u_{min}} \\
\label{eq:gamma_nuclear_bound}
 \normopnuc{O_uS} = \sum_{t\leq T} \lambda_t(O_u S) \leq 4n \norms{u_{max}} T \log T \\
\label{eq:gamma_hs_bound}
\frac{1}{4} \norms{u_{min}}^2 T^2 \leq \normophs{O_uS}^2 = \sum_{t\leq T} t \norms{u_t}^2 \leq \frac{1}{2} \norms{u_{max}}^2 T^2.
\end{align}
\end{lem}

The proof of this lemma uses the following auxiliary result:
Let $D_T: \RR^T \rightarrow \RR^{T-1}$ be the difference operator, 
 $(D_Tx)_t = x_{t+1} - x_t  \mbox{ for $t\leq T-1$}$. 
In the field of Finite Difference methods, the operator $D_TD_T^*$ is known as the Laplacian on the line, or as the discrete derivative with Dirichlet boundary conditions, and is well studied. The eigenvalues 
of $D_TD_T^{*}$ may be derived by a direct computation, and correspond to the roots of the Chebyshev polynomial of second kind of order $T$. In particular, the following holds:
\begin{lem}
\label{lem:D_singularvalues}
The operator $D_T$ has kernel of dimension $1$ and singular values 
$\lambda_l(D_T) = 2\sin \Brack{\frac{\pi (T-l) }{2T}}$
for $l=1,\ldots,T-1$.
\end{lem}
We refer to \cite{mitchell1980finite} for the proof 
of Lemma \ref{lem:D_singularvalues}. 
Next, in Lemma \ref{lem:Sinv_singular_values_expl}, we show that the inverse of the operator $S'$, defined in (\ref{eq:Sprime_def}), is a one dimensional perturbation of $D_T$, which implies bounds on singular values of $S'^{-1}$. 
\begin{lem}  
\label{lem:Sinv_singular_values_expl}
The singular values of $S'^{-1}_T$ satisfy 
\begin{equation}
2\sin \Brack{\frac{\pi (T-t)}{2(T+1)}} \leq \lambda_t (    S'^{-1}_T ) 
\leq 
2 \sin \Brack{\frac{\pi (T+1-t)}{2(T+1)}}
\end{equation}
for $1\leq t\leq T-1$, and 
\begin{equation}
\frac{1}{T} \leq \lambda_T (    S'^{-1}_T ) 
\leq 2 \sin \Brack{\frac{\pi }{2(T+1)}}. 
\end{equation}
\end{lem}
The proof of this is given in the next section. With the estimates of Lemma \ref{lem:Sinv_singular_values_expl}, the bounds in Lemma \ref{lem:OS_properties} follow.
Proof of Lemma \ref{lem:OS_properties}:
\begin{proof}
By Lemma \ref{lem:Sinv_singular_values_expl}, 
\begin{equation}
\label{eq:S'_lower}
\normop{S} \leq T 
\text{ and } \norms{S'_Tx} \geq \frac{1}{2} \norms{x} \text{ for all $x \in \RR^T$}.
\end{equation}

Note also that $S$ is by definition a collection of $n$ independent copies of $S'_T$, and therefore 
the spectrum of $S$ is that of $S'_T$, but each singular value is taken with multiplicity $n$.
In particular it follows that (\ref{eq:S'_lower}) holds also for $S$ itself. 
Since clearly $\normop{O_u} \leq \norms{u_{max}}$, the upper bound on $\normop{O_uS}$ in (\ref{eq:gamma_upper_lower}) 
 follows from (\ref{eq:S'_lower}).

For the lower bound, denote by $V'$ 
the orthogonal complement to the kernel of $O_u$,
$V' = \Brack{Ker(O_u)}^{\perp}$.
 Denote by 
$P_{V'} : \RR^{Tn} \rightarrow V'$ the orthogonal projection onto $V'$. 
We have in particular that $O_uS = O_u P_{V'}S$. Next, the operator 
$S$ maps the unit ball $B_{Tn}$ of $\RR^{Tn}$ into an ellipsoid $\mathcal{E}_{T}$, and by 
(\ref{eq:S'_lower}), we have $\half B_{Tn} \subset \mathcal{E}_{T}$. It therefore follows that 
\begin{equation}
\label{eq:lower_bound_ball_map}
\half B_{V_1} \subset (P_{V'}S)(B_{Tn}), 
\end{equation} 
where $B_{V_1}$ is the unit ball of $V_1$. It remains to observe that for every $x\in V_1$, 
we have 
\begin{equation}
\label{eq:lower_bound_O_u}
\norms{O_u x} \geq \norms{u_{min}} \norms{x}. 
\end{equation}
Combining (\ref{eq:lower_bound_ball_map}) and (\ref{eq:lower_bound_O_u}), we obtain the lower bound in (\ref{eq:gamma_upper_lower}). 

To derive (\ref{eq:gamma_nuclear_bound}), recall that the nuclear norm is 
sub-multiplicative with respect to the operator norm:
\begin{equation}
\label{eq:nuc_sub_mult}
\normopnuc{O_uS} \leq \normop{O_u} \cdot \normopnuc{S}  \leq \norms{u_{max}} \cdot \normopnuc{S}.
\end{equation}
This follows for instance from the characterization of the nuclear norm as trace dual of the operator norm \citep[Propositions IV.2.11, IV.2.12]{bhatia97matrix}.  Next, since the spectrum of $S$ is the spectrum of $S'$ taken with multiplicity $n$, we have 
\begin{equation}
\label{eq:nuc_S_to_S_prime}
\normopnuc{S} = n \normopnuc{S'},
\end{equation}
and it remains to bound the nuclear norm of $S'$.

Using the inequality 
\begin{equation}
\label{eq:sin_near_zero_ineq}
\sin \frac{\pi}{2} \alpha \geq \alpha   \mbox{ for all $\alpha \in [0,1]$}, 
\end{equation}
and Lemma \ref{lem:Sinv_singular_values_expl}, we have 
\begin{equation}
\label{eq:nuc_harmonic}
\normopnuc{S'} = T + 2\sum_{t\leq T-1} \sin^{-1}\Brack{\frac{\pi T-t}{2(T+1)}} \leq 2 T\sum_{t\leq T} \frac{1}{t+1}  \leq 
4T\log T. 
\end{equation}
Combining (\ref{eq:nuc_sub_mult}), (\ref{eq:nuc_S_to_S_prime}) and (\ref{eq:nuc_harmonic}), 
the inequality (\ref{eq:gamma_nuclear_bound}) follows.

It remains to estimate the Hilbert-Schmidt norm of 
$O_uS$, which can be done by a direct computation. Recall that for any operator $A:\RR^m \rightarrow \RR^m$ the Hilbert-Schmidt norm satisfies
\begin{equation}
    \normophs{A}^2 = \tr A^* A = \sum_{i\leq m} \inner{A^{*}A \phi_i}{\phi_i} = \sum_{i\leq m} \inner{A \phi_i}{A \phi_i} = 
    \sum_{i\leq m} \norms{A \phi_i}^2, 
\end{equation}
for any orthonormal basis $\phi_i$ in $\RR^{m}$. 
Let $e_{ti}$ be the standard basis vector in $\RR^{Tn}$ corresponding to coordinate $i\leq n$ at time $t\leq T$. 
Let $e_{t'}$, $t'\leq T$ denote the standard basis in $\RR^T$. Then 
\begin{flalign}
O_u S e_{ti} = 
\sum_{t' = t}^T u_{t'i} e_{t'},
\end{flalign}
where $u_{t'i}$ is the $i$-th coordinate of $u_{t'}$. It follows that
\begin{flalign}
\norms{O_u S e_{ti}}^2 = 
\sum_{t' = t}^T u_{t'i}^2 
\end{flalign}
and hence 
\begin{flalign}
\label{eq:gamma_hs_os_last_eq}
\normophs{O_uS}^2 =  \sum_{t\leq T, i\leq n} \norms{O_u S e_{ti}}^2 = 
\sum_{t\leq T} \sum_{t' = t}^T \norms{u_{t'}}^2 = \sum_{t\leq T} t \norms{u_t}^2.
\end{flalign}
The bounds (\ref{eq:gamma_hs_bound}) follow directly from (\ref{eq:gamma_hs_os_last_eq}).
\end{proof}

\section{Proof of Lemma \ref{lem:Sinv_singular_values_expl}}
\label{sec:Sinv_proof}
\begin{proof}
 Recall that the operator $S'^{-1}_T$ is given by 
  $S'^{-1}_T x = (x_1,x_2-x_1,\ldots,x_T-x_{T-1})$ and the operator $D=D_T:\RR^T \rightarrow \RR^{T-1}$ is given by $Dx = (x_2-x_1,\ldots,x_{T}-x_{T-1})$. 
Let $V = \text{span}\Set{e_2, \ldots, e_T}$ be the subspace spanned by all but the first coordinate. Let $P_V : \RR^T \rightarrow \RR^T$ be the projection onto $V$, i.e. a restriction to second to $T$'th coordinate. Observe that the action of $DP_V$ is equivalent to that of $S'^{-1}_{T-1}$. Therefore the singular values of $S'^{-1}_{T-1}$ are identical to those of $DP_V$. To obtain bounds on the singular values of $DP_V$, note that $(DP_V)^* DP_V = P_V D^* D P_V$ -- that is, 
$P_V D^* D P_V$ is a \emph{compression} of $D^*D$. Thus, by the Cauchy's Interlacing Theorem \cite[Corollary III.1.5]{bhatia97matrix}, 
\begin{equation}
\lambda_t(D^* D ) \geq   \lambda_t(P_V D^* D P_V)  \geq \lambda_{t+1}(D^* D )
\end{equation}
for all $t\leq T-1$. In conjunction with Lemma \ref{lem:D_singularvalues} this provides us with the estimates for all but the smallest singular value of $S'^{-1}_{T-1}$ (since $\lambda_{T}(D^* D ) = 0$). We therefore estimate 
$\lambda_{T-1}(S'^{-1}_{T-1}) = \lambda_T( D P_V)$ directly, by 
bounding the norm of $S'_{T-1}$. Indeed, for any $T$, by the  Cauchy-Schwartz inequality, 
\begin{flalign}
\norms{S'_{T} x}^2  = \sum_{t\leq T } \Brack{\sum_{i\leq t} x_i}^2 \leq \sum_{t\leq T } \Brack{\sum_{i\leq t} x_i^2}\Brack{\sum_{i\leq t} 1} \\
= \sum_{t\leq T } t \Brack{\sum_{i\leq t} x_i^2} 
\leq 
 \Brack{\sum_{t\leq T } t} \norms{x}^2
\leq \norms{x}^2 T^2. 
\end{flalign}
Thus we have $\normop{S'_T}\leq T$, which concludes the proof of the Lemma. 
\end{proof}

\section{Proof of Theorem  \ref{lem:R_properties}}
\label{sec:supp_r_properties}
\begin{proof}
The bound on $\normop{R}$ follows directly from 
the lower bound on the singular values of $O_uS$ in (\ref{eq:gamma_upper_lower}). 
Since $R$ is of rank $T$, the upper bounds on $\normophs{R}$ follow directly from the 
 $\normop{R}$ bound.

The lower bounds on $\normophs{R}^2$
and $\normophs{R^*R}^2$ follow from the upper bounds on the nuclear norm in Lemma  \ref{lem:OS_properties}. Note that this argument would not have worked if we only had upper bounds on the Hilbert-Schmidt norm, rather than the nuclear norm in Lemma \ref{lem:OS_properties}. 
 Denote $\gamma_i = \lambda_i(O_uS)$. From (\ref{eq:gamma_nuclear_bound}) in Lemma \ref{lem:OS_properties}, the number of  $\gamma_i$ that are larger than 
$4 n\norms{u_{max}} \log T$ satisfies 
\begin{equation}
\label{eq:set_of_large_gammas}
\#\Set{\gamma_i \setsep \gamma_i \geq 4 n\norms{u_{max}} \log T} \leq \frac{T}{2}. 
\end{equation}
Since there are total of $T$ singular values $\gamma_i$ overall, we can equivalently rewrite (\ref{eq:set_of_large_gammas}) as 
\begin{equation}
\label{eq:set_of_large_inv_gammas}
\#\Set{\gamma_i \setsep \gamma^{-1}_i \geq \frac{1}{4 n\norms{u_{max}} \log T}} \geq \frac{T}{2}. 
\end{equation}
This immediately implies the lower bounds on $\normophs{R}^2$ and $\normophs{R^*R}^2$. 
\end{proof}

\section{Proof of Theorem \ref{thm:main_thm} }
\label{sec:main_thm_proof}

The two main probabilistic 
tools that we use are the 
Hanson-Wright inequality \citep{hansonwright}, and a classical norm deviation inequality for sub Gaussian vectors, as follows: 

\begin{thm}[Hanson-Wright Inequality]  
\label{thm:hanson_wright}
Let $X=  (X_1,\ldots,X_m) \in \RR^m$ be a random vector such that the components $X_i$ are independent and $X_i \sim SG(\kappa)$ for all $i \leq m$.  Let A be an $m\times m$ matrix. Then, for every $t\geq 0$,
\begin{equation}
\label{eq:hanson_wright_statement_eq}
\Prob{\Abs{\inner{AX}{X} - \Exp{\inner{AX}{X}}}> t} 
\leq 
2 \exp \BBrack{ -c \min \Brack{\frac{t^2}{\kappa^4\normophs{A}^2},
                       \frac{t}{\kappa^2 \normop{A}}}}.    
\end{equation}
\end{thm}

In particular, recall that we are interested in concentration of $\norms{RY}^2$. We may write:
\begin{align}
\label{eq:sum_norm_decompoistion}
    \norms{RY}^2 = \norms{RO_uS h + Rz}^2  = 
    \norms{RO_uS h}^2  + \norms{Rz}^2 + \inner{h}{(RO_uS)^* Rz}. 
\end{align}
The deviations of the first two terms may be bounded 
via Theorem \ref{thm:hanson_wright}. 
For the third term, we use the following:
\begin{lem}
\label{lem:sg_norm_concentration}
For any $X \sim SG_m(\kappa)$  we have:
\begin{equation}
\Prob{ \norms{X} > 4\kappa \sqrt{m}+ t} 
\leq  \exp \Brack{ - \frac{c t^2}{\kappa^2}}.    
\end{equation}
\end{lem}
Note that in 
Lemma \ref{lem:sg_norm_concentration} we do not require 
the coordinates of $X$ to be independent. This will be 
important in what follows. Lemma \ref{lem:sg_norm_concentration} is standard and can be proved via covering number estimates of the Euclidean ball, see for instance \cite{hdp2018}, Section 4.4. 

In addition, the following observation is used throughout the text:
\begin{lem}
\label{lem:exp_of_a_norm_sq}
Let $A:\RR^n \rightarrow \RR^m$ be an operator, and let $h = (h_1,\ldots,h_m)$ have independent coordinates with $\Exp{h_i^2}=\sigma^2$.  Denote by $\Set{\lambda_i}_{i=1}^k$, $k\leq \min(m,n)$, the singular values of $A$. 
Then 
\begin{equation}
\Exp{\norms{Ah}^2} =  \sigma^2 \normophs{A}^2 = \sigma^2 \sum_{i=1}^k \lambda_i^2. 
\end{equation}
\end{lem}
The elementary proof is omitted. 

We now prove Theorem \ref{thm:main_thm}. 
\begin{proof}
Let $0<\delta<1$ be given. 
We first bound the deviations from the expectation for the first term in (\ref{eq:sum_norm_decompoistion}), $\norms{RO_uS h}^2$. 
We apply the Hanson-Wright inequality with $X=h$ and $A = (RO_uS)^*RO_uS$. By definition, $A$ has a single eigenvalue $1$ with multiplicity $T$. Thus clearly $\normophs{A}^2 = T$ and $\normop{A} = 1$. 

For an appropriate constant $c'>0$, set $t = c' \kappa^2 \sqrt{T} \log \frac{1}{\delta}$. Then,  
\begin{equation}
\label{eq:main_thm_first_term_final}
    \Prob{\Abs{  \norms{RO_uS h}^2 - T\sigma^2} \geq c'' \kappa^2 \sqrt{T \log \frac{1}{\delta}} } \leq \delta.
\end{equation}

The deviation of the second term in (\ref{eq:sum_norm_decompoistion}) is similarly bounded using $A = R^* R$. Recall that by Theorem \ref{lem:R_properties} we have $\normop{R} \leq 2\norms{u_{min}}^{-1}$, and note that 
$\normophs{R^*R}^2 \leq T\normop{R}^4 \leq c T\norms{u_{min}}^{-4}$. Set $t = c'\kappa^2 \norms{u_{min}}^{-2} \sqrt{T} \log \frac{1}{\delta}$.  With this choice it follows that both terms in the minimum in (\ref{eq:hanson_wright_statement_eq}) are larger than $c\log \frac{1}{\delta}$ and we have 
\begin{align}
\label{eq:main_thm_second_term_final}
\Prob{\Abs{  \norms{Rz}^2 - \eta^2 \normophs{R}^2} \geq c'\kappa^2 \norms{u_{min}}^{-2} \sqrt{T} \log \frac{1}{\delta} } \leq \delta. 
\end{align}

Finally, we bound the third term in (\ref{eq:sum_norm_decompoistion}). 
Denote by $D$ the event 
\begin{align}
  D =  \BBrack{\norms{(RO_uS)^*Rz} > c \norms{u_{min}}^{-1} \Brack{\kappa \sqrt{T} +  \kappa \sqrt{\log \frac{1}{\delta}}}   }, 
\end{align}
and by $E$ the event 
\begin{equation}
    E = \BBrack{\Abs{\inner{h}{(RO_uS)^*Rz}} > c' \norms{u_{min}}^{-1} \Brack{\kappa \sqrt{T} +  \kappa \sqrt{\log \frac{1}{\delta}}   } \kappa \sqrt{\log \frac{1}{\delta}}}. 
\end{equation}
By Lemma \ref{lem:sg_norm_concentration} applied to $z$, and using the fact that 
$\normop{(RO_uS)^*R}\leq 2 \norms{u_{min}}^{-1}$,
\begin{align}
\label{eq:main_thm_D_bound}
  \Prob{ D   } \leq 
    \Prob{ \norms{z} > c\kappa \sqrt{T} + c \kappa \sqrt{\log \frac{1}{\delta}}   } \leq \delta. 
\end{align}
Next, using independence of $h,z$ and $h \sim SG_{Tn}(\kappa)$,
\begin{align}
\label{eq:main_th_E_bound}
    \Prob{ E \;\middle|\; D^c} \leq  \delta 
\end{align}
where $D^c$ is the complement of $D$. Therefore, combining 
(\ref{eq:main_thm_D_bound}) and (\ref{eq:main_th_E_bound}), 
\begin{align}
\label{eq:main_thm_thord_term_final}
    \Prob{ E } = \Prob{D}\Prob{E|D} + \Prob{D^c}\Prob{E|D^c} \leq 2 \delta.
\end{align}
Combining (\ref{eq:main_thm_first_term_final}), (\ref{eq:main_thm_second_term_final}) and  (\ref{eq:main_thm_thord_term_final}), we obtain via the union bound:
\begin{equation}
\label{eq:main_thm_dev_for_R}
\Prob{\Abs{\frac{\norms{RY}^2}{T} - \Brack{\sigma^2 + \frac{\normophs{R}^2}{T}\eta^2}} \geq 
c \frac{\Brack{1+\kappa^2}\Brack{1+\norms{u_{min}}^{-2}}\log \frac{1}{\delta}}{\sqrt{T}}
} \leq 4 \delta. 
\end{equation}
Similarly, we obtain a bound for the equations involving $R'$: 
\begin{equation}
\label{eq:main_thm_dev_for_R_tag}
\Prob{\Abs{\frac{\norms{R'Y}^2}{p} - \Brack{\sigma^2 + \frac{\normophs{R'}^2}{p}\eta^2}} \geq 
c \frac{\Brack{1+\kappa^2}\Brack{1+\norms{u_{min}}^{-2}}\log \frac{1}{\delta}}{\sqrt{p}}
} \leq 4 \delta. 
\end{equation}
The only difference in the derivation of (\ref{eq:main_thm_dev_for_R_tag}) compared to (\ref{eq:main_thm_dev_for_R}) is in the application of Lemma \ref{lem:sg_norm_concentration}. In the later case, 
to replace $\sqrt{T}$ with $\sqrt{p}$ in (\ref{eq:main_thm_D_bound}), we apply Lemma \ref{lem:sg_norm_concentration} with $X= P_V z$ rather than with $X=z$, where $P_V$ is the projection onto the range of $R'$ -- a $p$-dimensional space. Note that $P_V z$ does not necessarily have a structure of $p$ independent coordinates, but is sub Gaussian and isotropic. Therefore Lemma \ref{lem:sg_norm_concentration} still applies. 

\end{proof}

\section{Proof of Corollary \ref{cor:estimator_bounds}}
\label{sec:coroll_proof}
We now turn to prove Corollary \ref{cor:estimator_bounds}.

\begin{proof}
Denote 
\begin{equation}
    E(a) = \frac{\Brack{1+\kappa^2}\Brack{1+\norms{u_{min}}^{-2}}\log \frac{1}{\delta}}{a}. 
\end{equation}

Using (\ref{eq:main_thm_dev_for_R_tag}) and (\ref{eq:main_thm_dev_for_R}) we may write 
\begin{flalign}
    \frac{\norms{RY}^2}{T} + e_1 =  \sigma^2 + \frac{\normophs{R}^2}{T}\eta^2,  \\
    \frac{\norms{R'Y}^2}{p} + e_2 = \sigma^2 + \frac{\normophs{R'}^2}{p}\eta^2,   
\end{flalign}
where $e_1,e_2$ are error terms such that $\Abs{e_1}\leq E(\sqrt{T})$ and $\Abs{e_2}\leq E(\sqrt{p})$ holds with probability at least $1-8\delta$. 

It follows that 
\begin{flalign}
    \eta^2 = \Brack{\frac{\norms{R'Y}^2}{p} - \frac{\norms{RY}^2}{T}}\Brack{\frac{\normophs{R'}^2}{p} - \frac{\normophs{R}^2}{T}}^{-1} + \Brack{e_2-e_1}\Brack{\frac{\normophs{R'}^2}{p} - \frac{\normophs{R}^2}{T}}^{-1}  
\end{flalign}
and 
\begin{flalign}
    \frac{p\normophs{R}^2}{T\normophs{R'}^2}\frac{\norms{R'Y}^2}{p} + \frac{p\normophs{R}^2}{T\normophs{R'}^2}e_2 = \frac{p\normophs{R}^2}{T\normophs{R'}^2}\sigma^2 + \frac{\normophs{R}^2}{T}\eta^2.   
\end{flalign}
Thus 
\begin{flalign}
\sigma^2 = \Brack{1- \frac{p\normophs{R}^2}{T\normophs{R'}^2}}^{-1} \Brack{\frac{\norms{RY}^2}{T} - \frac{\norms{R'Y}^2}{T} \frac{\normophs{R}^2}{\normophs{R'}^2} } \nonumber \\ 
+\Brack{1- \frac{p\normophs{R}^2}{T\normophs{R'}^2}}^{-1}\Brack{e_1 - \frac{p\normophs{R}^2}{T\normophs{R'}^2} e_2}.
\end{flalign}
It remains to observe that 
\begin{align}
\Brack{\frac{\normophs{R'}^2}{p} - \frac{\normophs{R}^2}{T}}^{-1} & = 
\Brack{1- \frac{p\normophs{R}^2}{T\normophs{R'}^2}}^{-1} \frac{p}{\normophs{R'}^2}  \\
& \leq 
c \Brack{1- \frac{p\normophs{R}^2}{T\normophs{R'}^2}}^{-1} n^2 \norms{u_{max}}^2 \log^2 T, 
\end{align}
where the inequality follows from eq. 
(\ref{eq:set_of_large_inv_gammas}) in the proof of Theorem \ref{lem:R_properties}. 
\end{proof}

\section{Thresholding Gap Analysis}
\label{sec:gap_analysis}
\label{gap_bound_general_proof}

The main result of this section is the following Proposition.

\begin{prop} 
\label{lem:gap_bound_general}
Given the sequence $\Set{u_t}_{t=1}^T \subset \RR^n$, define the scalar
sequence $\tilde{u}_{t} = \sum_{i \leq n} u_{ti} $
and set 
\begin{equation}
    \norms{\tilde{u}_{min}} = \min_t \Abs{\tilde{u}_{t}} \text{ and }
    \norms{\tilde{u}_{max}} = \max_t \Abs{\tilde{u}_{t}}.
\end{equation}
Then for $p = \frac{1}{4}T$, 
\begin{equation}
\label{eq:lem_gap_bound_general_statement}
    \frac{\normophs{R'}^2}{p} - \frac{\normophs{R}^2}{T} \geq 
    c \Brack{\frac{ \norms{\tilde{u}_{min}}}{n \norms{u_{max}}}}\frac{\normophs{R}^2}{T}.
\end{equation}
\end{prop}

The general idea behind the proof of Proposition \ref{lem:gap_bound_general} is to show that for $n=1$, the ratio 
$\Brack{1- \frac{p\normophs{R}^2}{T\normophs{R'}^2}}^{-1}$ can be controlled. This is done in Lemmas \ref{lem:cone_bound} and \ref{lem:gap_bound} below. In particular, Lemma \ref{lem:cone_bound} is a general statement about integrals of monotone real functions under certain order constraints, and Lemma \ref{lem:gap_bound} provides a relation of the spectrum of $R$ to that of $S'^{-1}$. It is then shown that for arbitrary $n$, $O_uS$ contains a certain copy of an $n=1$-dimensional operator with parameters $\tilde{u}_t$, which implies the bounds.

For the case $n=1$, stated in Lemma \ref{lem:gap_bound},  the argument consists of showing that the spectrum of $R$ is upper and lower bounded by appropriately decaying functions, and therefore can not be ``too constant''. 
We first obtain general estimates for the integrals of such upper and lower bounded functions in the following Lemma:
\begin{lem}
\label{lem:cone_bound}
Let $f:[0,1] \rightarrow \RR$ be a monotone non-increasing function such that for all $x \in [0,1]$,
\begin{equation}
\label{eq:lem_cone_bound_constraint}
  \Brack{1-x}^2 \leq f(x) \leq M (1-x)^2,    
\end{equation}
for some $M\geq 1$. Set $t_0 = \frac{1}{4}$. Then 
\begin{equation}
    r(f) := \frac{\frac{1}{t_0}\int_{0}^{t_0} f(x) dx}{\int_{0}^{1} f(x) dx} \geq 1 + \frac{c}{\sqrt{M}}. 
\end{equation}
\end{lem}
\begin{proof}
Denote 
\begin{equation}
    I(a,b,f) = \int_{a}^b f(x) dx. 
\end{equation}
Write 
\begin{equation}
    r(f) = \frac{t_0^{-1} I(0,t_0,f)}{I(0,t_0,f) + I(t_0,1,f)} = \frac{t_0^{-1}}{1 + \frac{I(t_0,1,f)}{I(0,t_0,f)}}
\end{equation}
and set $v:= f(t_0)$. Then, among all $f$ that satisfy (\ref{eq:lem_cone_bound_constraint}) and $f(t_0)=v$, 
$r(f)$ is minimized on $f$ with maximal $I(t_0,1,f)$ and minimal $I(0,t_0,f)$. Due to the form of the constraint (\ref{eq:lem_cone_bound_constraint}) and monotonicity, this minimizer is given by  
\begin{equation}
    \tilde{f_v}(x) = \begin{cases}
    (1-x)^2 & x\in [0,t_v^{-}] \\
    v  & x \in [t_v^{-},t_v^{+}] \\
    M(1-x)^2  & x \in [t_v^{+},1], 
    \end{cases}
\end{equation}
where $t_v^{-} := \max \BBrack{0, 1 - \sqrt{v}}$ and 
$t_v^{+} := 1 - \sqrt{\frac{v}{M}}$.
Our problem is therefore now reduced from a minimization of $r(f)$ over the function space to a problem of minimizing $r(v) := r(\tilde{f_v})$ over a single scalar variable $v$.

To this end, first note that we can assume w.l.o.g that $M\geq (1-t_0)^{-2}$.  Indeed, for larger $M$, the lower bound on $r(f)$ can only become smaller, since $r(f)$ would be minimized over a larger set. By construction, the value $v$ must satisfy $(1-t_0)^{2} \leq v \leq M (1-t_0)^{2}$, and for $M\geq (1-t_0)^{-2}$, the value $v=1$ satisfies these inequalities. 

Next, by direct computation (taking the derivative in $v$) one verifies that for any $M\geq (1-t_0)^{-2}$, 
$r(v)$ as a function of $v$ is minimized at $v=1$. It remains to observe that  $r(1) = 1 - \frac{1}{3\sqrt{M}}$, which yields the statement of the Lemma.
\end{proof}

We can now treat the $n=1$ case.
\begin{lem} 
\label{lem:gap_bound}
Consider the case $n=1$, and $p=\frac{1}{4}T$. Then 
\begin{equation}
\label{eq:lem_gap_bound_statement}
    \frac{\normophs{R'}^2}{p} - \frac{\normophs{R}^2}{T} \geq 
    c \Brack{\frac{\norms{u_{min}}}{\norms{u_{max}}}}\frac{\normophs{R}^2}{T}.
\end{equation}
\end{lem}
\begin{proof}
For any two operators $A,B$ we have 
\begin{equation}
\label{eq:singular_prod_norm_compare}
    \lambda_i(AB) \leq \lambda_i(A)\normop{B} \text{ and } \lambda_i(BA) \leq \lambda_i(A)\normop{B},
\end{equation}
see \cite{bhatia97matrix,gohberg1969introduction}.
Note that in the case $n=1$ and $\norms{u_{min}}>0$, 
$O_uS$ is invertible. Thus the singular values of $O_uS'$ satisfy 
\begin{equation}
    \norms{u_{min}} \lambda_i(S') \leq \lambda_i(O_uS') \leq   \norms{u_{max}} \lambda_i(S'),
\end{equation}
where the first inequality follows by applying (\ref{eq:singular_prod_norm_compare}) to $\Brack{O_uS'}^{-1}$, and the second by considering $O_uS'$ itself. Equivalently,  for $i=1,\ldots,T$
\begin{equation}
    \norms{u_{max}}^{-1} \lambda_{T+1-i}(S'^{-1}) \leq \lambda_{T+1-i}(R) \leq   \norms{u_{min}}^{-1} \lambda_{T+1-i}(S'^{-1}),
\end{equation}
and using Lemma \ref{lem:Sinv_singular_values_expl}, 
\begin{equation}
    c \frac{  \norms{u_{max}}^{-1}i}{T+1} \leq \lambda_{T+1-i}(R) \leq   \norms{u_{min}}^{-1} \frac{ \pi \norms{u_{min}}^{-1}i}{T+1}.
\end{equation}
By changing the index and taking squares, we have 
\begin{equation}
    c \norms{u_{max}}^{-2} \Brack{1-\frac{i}{T+1}}^2 \leq \lambda_{i}^2(R) \leq   c' \norms{u_{min}}^{-2} \Brack{1-\frac{i}{T+1}}^2.
\end{equation}
Now, using an elementary discretization argument, for $p=\frac{1}{4}T$ by Lemma \ref{lem:cone_bound} it follows that 
\begin{equation}
    \frac{\normophs{R'}^2/p}{\normophs{R}^2/T} \geq 1 + c \frac{\norms{u_{max}}}{\norms{u_{min}}},
\end{equation}
which implies (\ref{eq:lem_gap_bound_statement}).
\end{proof}

Finally, we prove Proposition \ref{lem:gap_bound_general}.
\begin{proof}
As discussed earlier, the key to a statement such as (\ref{eq:lem_gap_bound_general_statement}) is to show that the spectrum of $R$ is non constant, and in particular has enough \emph{small} singular values.   This is equivalent to providing appropriate lower bounds on the spectrum of $O_uS$. Here we derive such bounds by comparison with an $n=1$ case as follows: 
Let $V \subset \RR^{Tn}$ be a $T$ dimensional subspace spanned by vectors for which for every time $t$, all coordinates at time $t$ are identical. Formally, for $x \in V$, for every $t$ we require that $x_{(t-1)n+i} = x_{(t-1)n+j}$ for all $i,j \leq n$.

Observe that the restriction of $O_uS$ to $V$, the operator $O_u S P_V$, acts equivalently to the $n=1$  case operator defined by $O_{\tilde{u}}S'$. In particular, similarly to the argument in Lemma \ref{lem:gap_bound}, it follows that 
\begin{equation}
    \norms{\tilde{u}_{min}} \lambda_i(S') \leq \lambda_i(O_{\tilde{u}}S'). 
\end{equation}
Moreover, note that 
$\Brack{O_u S P_V}^*O_u S P_V = 
P_V S^* O_u^* O_u S P_V$ and 
$\Brack{O_u S}^*O_u S$ are non-negative operators and 
$\Brack{O_u S}^*O_u S \geq \Brack{O_u S P_V}^*O_u S P_V$ (that is, 
$\Brack{O_u S}^*O_u S - \Brack{O_u S P_V}^*O_u S P_V$ is non-negative). 
It follows that for all $i\leq T$, 
\begin{equation}
\lambda_i(O_uS) \geq \lambda_i(O_uSP_V).  
\end{equation}
The upper bounds on the spectrum 
$\lambda_i(O_uS)$ may again be obtained via (\ref{eq:singular_prod_norm_compare}). Indeed, 
\begin{equation}
\lambda_i(O_uS) \leq \norms{u_{max}} \lambda_i(S) \leq \norms{u_{max}} \lambda_{\ceil*{\frac{i}{n}}}(S'),
\end{equation}
where the first inequality follows from (\ref{eq:singular_prod_norm_compare}) while the second is due to the multiplicity $n$ of each singular value of $S'$ in $S$.  The rest of the argument proceeds as in Lemma \ref{lem:gap_bound}. 
\end{proof}

\section{Missing Values in STVE}
\label{sec:missing_values}
We now discuss the treatment of missing values in STVE. 
Recall that our starting point is the vector form of the system,  (\ref{eq:main_generative_equation}), which we rewrite here:
\begin{equation}
\label{eq:base_generative_supp}
Y = O_u S h + z,
\end{equation}
where $h \in \RR^{Tn}$ and $z\in \RR^T$ are sub Gaussian vectors and 
$Y = (y_1, \ldots, y_T) \in \RR^T$  is the observation vector. 
Suppose that $M$ out of $T$ observation values are missing, at times $t_1,\ldots,t_M$.
Set $T' = T - M$ and 
define a projection operator $P_A : \RR^{T} \rightarrow \RR^{T}$ as the operator 
that omits the coordinates $t_1,\ldots,t_M$. Formally, let $e(t)$, $t\leq T$, be the standard basis vector in $\RR^T$, with $1$ at coordinate $t$ and zeros elsewhere. 
Then 
\begin{equation}
P_A(e(t)) = 
\begin{cases}
    0 & \text{if } t \in \Set{t_1,\ldots,t_M}, \\
    e(t) & \text{ otherwise.}
\end{cases}
\end{equation}
We can then rewrite (\ref{eq:base_generative_supp}) as 
\begin{equation}
P_A Y = P_A O_u S h + P_A z. 
\end{equation}
Note that the vector $P_A Y$ contains only available, non-missing values of $y_t$. 
Similarly to the case with no missing values, we define $R$ as the Moore-Penrose inverse of $P_A O_u S$. 
Then we have 
\begin{equation}
R P_A Y = R P_A O_u S h + R P_A z.
\end{equation}
Note that since a Moore-Penrose inverse of $R$ only acts on the image of $R$, 
we have $R P_A = R$ and therefore 
\begin{equation}
R Y = R  O_u S h + R z
\end{equation}
similarly to the case with no missing values. The only difference here will be that $R$ will now have $T'$ 
rather than $T$ non-zero singular values. $R'$ will be defined similarly.
The analysis of the spectrum of $R$ concerns only non-zero singular values of $R$ and holds with no 
change other than that $T$ should be replaced by $T'$.  Consequently, 
the whole approach works identically with $T$ replaced by $T'$ everywhere, and in the bounds of Theorems 
\ref{thm:main_thm} and \ref{lem:R_properties} in particular.

\section{Electricity Consumption and Temperatures Data}
\label{sec:electicity_additional_details}

The raw data in \cite{electric_data_paper} contains 
electricity consumption levels for 20 nearby areas, also referred to as  zones. Each zone had slightly different consumption characteristics. All figures in Section \ref{sec:electricity_experiments} refer to Zone 1. However, the results are similar for all zones. 
In particular, while Figure \ref{fig:electric_zone_0_err} shows the (smoothed) prediction errors of each method over time, the total error on the test set, normalized by the number of days, for each method and each zone, is shown in Figure \ref{fig:electric_total_performance}. 
 As with zone 1, adaptive methods perform better for 
all zones.  STVE based estimator is better than MLE in half the cases, 
but the differences between STVE and MLE performance are negligible compared to errors in the 
stationary regression or OG.


In the rest of this section we discuss the preprocessing of the data. 
The raw data contains 
hourly consumption levels for the 20 zones, and hourly 
temperatures from 11 weather stations with unspecified 
locations in the region.  Here we are only interested 
in the total daily consumption -- the total consumption over 
all hours of the day -- for each zone. For each station we 
also consider the average daily temperature for the day. 
Moreover, since the average daily temperatures at different 
stations are strongly correlated ($\rho \geq 0.97$ for all 
pairs of stations), we only consider a single daily number, 
$v_t$ -- an average temperature across all hours of the day 
and all stations. 

From the total of about 234 weeks in the data, 
9 non-consecutive weeks have missing loads 
(consumption) data, 4 in the first half of the 
period (train set), and 5 in the second half. 
For the purposes of training the stationary 
regression, we excluded the 
data points with missing values from the train 
set. The STVE, MLE and OG methods were trained with missing values. 
The prediction results of all methods 
were evaluated only at the points where actual 
values were known. 

The temperatures and all the loads are normalized to have zero-mean and standard deviation 1 on the first half 
of the data (train set).

The MLE estimates of the parameters $\sigma^2, \eta^2$ 
were obtained using the DLM package of the R environment  \citep{rdlm}. The package uses an L-BGFS-B algorithm with 
numerically approximated derivatives as the underlying 
optimization procedure.


\begin{figure}
\centering
\subcaptionbox{
Load over time, with 6 month (orange) and year (green) moving average smoothing.
\label{fig:electric_consumption_trend}}{\includegraphics[width=0.5\textwidth]{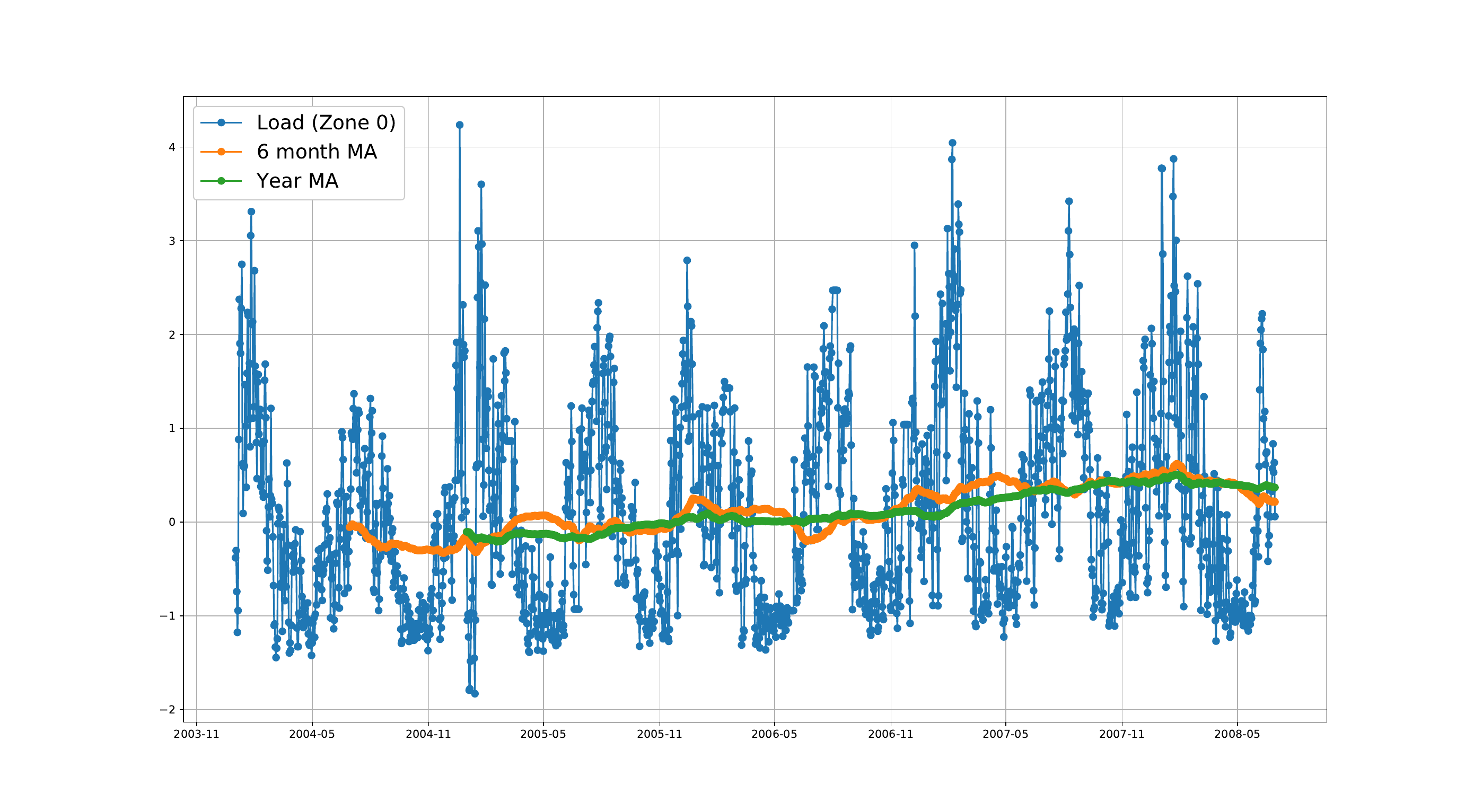}}
\hfill
\subcaptionbox{
Total average prediction error, for each zone.
\label{fig:electric_total_performance}}{
\includegraphics[width=0.45\textwidth]{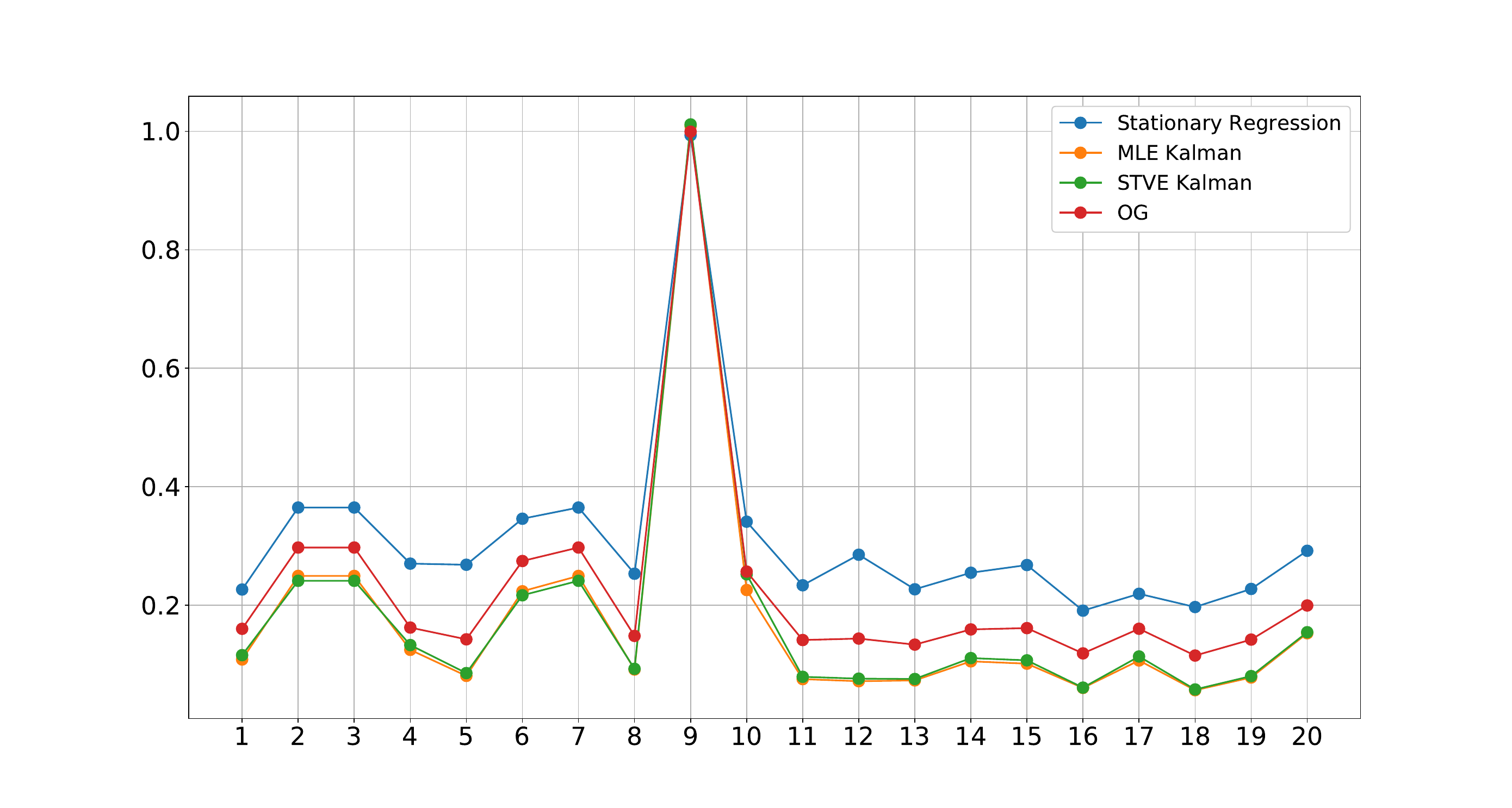}
}%
\caption{Supplementary Experiment Figures}
\end{figure}

One immediate reason for the dependence of the load on temperature to change 
with time (as shown in Figure \ref{fig:electic_staionary_first_second} for Zone 1) is that the load simply grows with time 
(assuming the temperatures do not exhibit a trend of the 
same magnitude). Indeed, it is easy to verify that there is an upward trend in the 
overall load, as shown in Figure \ref{fig:electric_consumption_trend}. However, the upward trend is non-uniform in temperature, since the parabolas in Figure \ref{fig:electic_staionary_first_second} are not a shift of each other by a constant.

Finally, in Figure  \ref{fig:electric_total_performance} average prediction errors  are shown for each method and for each zone in the dataset. Zone 9 is known to be an industrial zone,
\citep{electric_data_blog}, where the load does not significantly depend on the temperature, 
hence the higher error rates for all the methods. For the other zones the situation is similar to Zone 1, with STVE and MLE performing comparably, and better than either OG or the Stationary regression. 
\end{document}